\documentclass[letterpaper,11pt]{article}
\usepackage[utf8]{inputenc}
\usepackage[dvipsnames]{xcolor}
\usepackage{standalone}
\usepackage{tikz}
\usepackage{tikz-3dplot}
\usetikzlibrary{decorations.markings}
\usepackage[T1]{fontenc}
\usepackage[mathscr]{eucal}
\usepackage{prettyref}
\usepackage[compact]{titlesec}
\usepackage{mathpazo,tgcursor,tgtermes}
\usepackage{textcomp}
\usepackage[margin=1in]{geometry}
\usepackage{booktabs}
\usepackage{amsmath}
\usepackage{authblk}
\usepackage{amssymb}
\usepackage{graphicx}
\usepackage{mathtools}
\usepackage{tabulary}

\usepackage{xparse}
\usepackage{comment}
\usepackage{color}
\usepackage{amsmath}
\usepackage{amsthm}
\usepackage[linesnumbered, boxed]{algorithm2e}
\usepackage{tikz}
\usepackage{varwidth}
\usepackage[pagebackref,bookmarksnumbered,bookmarksopen,plainpages=false,pdfpagelabels,%
unicode,
breaklinks,colorlinks,citecolor=blue,linkcolor=blue,hyperindex]{hyperref}

\hypersetup{
colorlinks=true,
urlcolor=blue,
linkcolor=blue,
citecolor=OliveGreen}

\makeatletter 
\let\orgdescriptionlabel\descriptionlabel
\renewcommand*{\descriptionlabel}[1]{%
  \let\orglabel\label
  \let\label\@gobble
  \phantomsection
  \protected@edef\@currentlabel{#1}%
  \let\label\orglabel
  \orgdescriptionlabel{#1}%
}
\makeatother
\newtheorem{theorem}{Theorem}[section]
\newtheorem{assumption}[theorem]{Assumption}
\newtheorem{lemma}[theorem]{Lemma}
\newtheorem{remark}[theorem]{Remark}

\newtheorem{definition}[theorem]{Definition}
\newtheorem{corollary}[theorem]{Corollary}

\newcommand{\R}{\mathbb{R}}
\newcommand*{\size}[1]{\left|#1\right|}

\newcommand{\E}[2][]{\mathbb{E}_{#1}\left[#2\right]}
\newcommand{\norm}[1]{\left\Vert #1 \right\Vert}

\newcommand{\TD}{\operatorname{TD}}
\newcommand{\Q}{\operatorname{Q}}
\newcommand{\Sarsa}{\operatorname{SARSA}}
\newcommand{\LSTD}{\operatorname{LSTD}}
\newcommand{\LSPE}{\operatorname{LSPE}}

\title{Reinforcement Learning under Model Mismatch}

\author[1]{Aurko Roy}
\affil[1]{Google \thanks{Work done while at Georgia Tech}
  \textit{Email:}~\texttt{aurkor@google.com}}

 \author[2]{Huan Xu}
 \affil[2]{ISyE, Georgia Institute of Technology,
 Atlanta, GA,
 USA.
 \textit{Email:}~\texttt{huan.xu@isye.gatech.edu}
 } 

\author[2]{Sebastian Pokutta}
 \affil[2]{ISyE, Georgia Institute of Technology,
 Atlanta, GA,
 USA.
 \textit{Email:}~\texttt{sebastian.pokutta@isye.gatech.edu}
}
\begin{document}

\maketitle
\begin{abstract}
  We study reinforcement learning under \emph{model misspecification},
  where we do not have access to the true environment but only to a
  reasonably close approximation to it. We address this problem by
  extending the framework of robust MDPs of \cite{bagnell2001solving,
    nilim2003robustness, iyengar2005robust} to the \emph{model-free}
  Reinforcement Learning setting, where we do not have access to the
  model parameters, but can only sample states from it.  We define
  \emph{robust versions} of \(\Q\)-learning, \(\Sarsa\), and
  \(\TD\)-learning and prove convergence to an approximately optimal
  robust policy and approximate value function respectively.  We scale
  up the robust algorithms to large MDPs via function approximation
  and prove convergence under two different settings. We prove
  convergence of robust approximate policy iteration and robust
  approximate value iteration for linear architectures (under mild
  assumptions). 
We also define a
  robust loss function, the \emph{mean squared robust projected
    Bellman error} and give stochastic gradient descent algorithms
  that are guaranteed to converge to a local minimum.
\end{abstract}

\section{Introduction}
Reinforcement learning is concerned with learning a good policy 
for sequential decision making problems modeled as a Markov 
Decision Process (MDP), via interacting with the environment
\cite{sutton1998reinforcement, puterman2014markov}.
In this work we address the problem of reinforcement learning
from a \emph{misspecified model}. As a motivating example, consider
the scenario where the problem of interest is not directly accessible, but
instead the agent can interact with a simulator whose dynamics is 
reasonably close to the true problem. Another plausible application is 
when the parameters of the model may evolve over time
but can still be reasonably approximated by an MDP.

To address this problem we use the framework of \emph{robust MDPs}
which was proposed by \cite{bagnell2001solving,
nilim2003robustness, iyengar2005robust}
to solve the planning problem 
under model misspecification. 
The robust MDP framework considers a class of models
and finds the robust optimal policy which is a policy 
that performs best under 
the worst model. 
It was shown by \cite{bagnell2001solving,
nilim2003robustness, iyengar2005robust} that the robust
optimal policy satisfies the \emph{robust Bellman
equation} which naturally leads to exact dynamic 
programming algorithms to find an optimal policy.
However, this approach is model dependent and 
does not immediately generalize to the model-free 
case where the parameters of the model are unknown.

Essentially, reinforcement learning is a \emph{model-free} 
framework to solve 
the Bellman equation using samples. Therefore, to learn
policies from misspecified models, we develop sample based
methods to solve the \emph{robust} Bellman equation. In particular, 
we develop robust versions of 
classical reinforcement learning algorithms such as 
\(\Q\)-learning, \(\Sarsa\), and \(\TD\)-learning and prove 
convergence to an approximately optimal policy 
under mild assumptions on the discount factor. 
We also show that the nominal versions of these iterative algorithms
converge to policies that may be arbitrarily worse compared to the
optimal policy. 

We also scale up these robust algorithms 
to large scale MDPs via function approximation, 
where we prove convergence under two different settings. 
Under a technical assumption similar to 
\cite{bertsekas2009projected,tamar2014scaling}
we show convergence of robust approximate policy iteration 
and value iteration algorithms for linear architectures.
We also study function approximation with nonlinear
architectures, by defining an appropriate 
\emph{mean squared robust projected Bellman error} (MSRPBE) loss
function, which is a generalization of the 
mean squared projected Bellman error (MSPBE) loss function of 
\cite{sutton2009convergent, sutton2009fast, bhatnagar2009convergent}.
We propose robust versions of stochastic gradient descent algorithms
as in \cite{sutton2009convergent, sutton2009fast, bhatnagar2009convergent}
and prove convergence to a local minimum under some assumptions 
for function approximation with arbitrary smooth functions.

\paragraph{Contribution.} In summary we have the following contributions:
\begin{enumerate}
\item We extend the robust MDP framework of \cite{bagnell2001solving,
    nilim2003robustness, iyengar2005robust} to the \emph{model-free}
  reinforcement learning setting. We then define robust versions of
  \(\Q\)-learning, \(\Sarsa\), and \(\TD\)-learning and prove
  convergence to an approximately optimal robust policy.

\item We also provide robust reinforcement learning algorithms for the
  function approximation case and prove convergence of robust
  approximate policy iteration and value iteration algorithms for
  linear architectures. We also define the MSRPBE loss function which
  contains the robust optimal policy as a local minimum and we derive
  stochastic gradient descent algorithms to minimize this loss
  function as well as establish convergence to a local minimum in the
  case of function approximation by arbitrary smooth functions.

\item Finally, we demonstrate empirically the improvement in
  performance for the robust algorithms compared to their nominal
  counterparts. For this we used various Reinforcement Learning test
  environments from OpenAI \cite{brockman2016openai} as benchmark to
  assess the improvement in performance as well as to ensure reproducibility and
  consistency of our results.
\end{enumerate}

\paragraph{Related Work.}
Recently, several approaches have been proposed to address
model performance due to parameter uncertainty for 
Markov Decision Processes (MDPs). A Bayesian approach
was proposed by \cite{shapiro2002minimax} which requires perfect 
knowledge of the prior distribution on transition matrices.
Other probabilistic and 
risk based settings were studied by 
\cite{delage2010percentile, wiesemann2013robust, tamar2014optimizing}
which propose various mechanisms to incorporate percentile risk
into the model.
A framework for robust MDPs was first proposed by
\cite{bagnell2001solving, nilim2003robustness, 
iyengar2005robust} who consider the transition matrices to lie
in some \emph{uncertainty set} and proposed a dynamic programming algorithm
to solve the robust MDP. Recent work by \cite{tamar2014scaling} extended 
the robust MDP framework to the function approximation setting
where under a technical assumption the authors prove convergence
to an optimal policy for linear architectures. 
Note that these algorithms for robust MDPs do not
readily generalize to the \emph{model-free}
reinforcement learning setting 
where the parameters of the environment are not 
explicitly known.

For reinforcement learning in the non-robust \emph{model-free}
setting, several iterative algorithms such as \(\Q\)-learning,
\(\TD\)-learning, and \(\Sarsa\) are known to converge to an optimal
policy under mild assumptions, see \cite{bertsekas1995neuro} for a
survey.  Robustness in reinforcement learning for MDPs was studied by
\cite{morimoto2005robust} who introduced a robust learning framework
for learning with disturbances. Similarly, \cite{pinto2017robust}
also studied learning in the presence of an adversary who might apply
disturbances to the system.  However, for the algorithms proposed in
\cite{morimoto2005robust,pinto2017robust} no theoretical guarantees
are known and there is only limited empirical evidence. Another recent
work on robust reinforcement learning is \cite{lim2013reinforcement},
where the authors propose an online algorithm with certain transitions
being stochastic and the others being adversarial and the devised
algorithm ensures low regret.

For the case of reinforcement learning with large MDPs using function
approximations, theoretical guarantees for most \(\TD\)-learning based
algorithms are only known for linear architectures
\cite{bertsekas2011approximate}. Recent work by
\cite{bhatnagar2009convergent} extended the results of
\cite{sutton2009convergent, sutton2009fast} and proved that a
stochastic gradient descent algorithm minimizing the \emph{mean
  squared projected Bellman equation} (MSPBE) loss function converges
to a local minimum, even for nonlinear architectures. However, these
algorithms do not apply to robust MDPs; in this work we extend
these algorithms to the robust setting.

\section{Preliminaries}\label{sec:prelim} 
We consider an infinite horizon Markov Decision Process (MDP) \cite{puterman2014markov}
with finite state space \(\mathcal{X}\) of size \(n\) and finite action space 
\(\mathcal{A}\) of size \(m\). At every time step \(t\) the agent is in 
a state \(i \in \mathcal{X}\)  
and can choose an action \(a \in \mathcal{A}\) incurring a cost \(c_t(i, a)\).
We will make the standard assumption that future cost is discounted, see e.g., 
\cite{sutton1998reinforcement}, 
with a discount factor \(\vartheta < 1\) applied 
to future costs, i.e.,
\(
 c_t(i, a) \coloneqq \vartheta^t c(i, a),
\)
where  \(c(i, a)\) is a fixed constant independent
of the time step \(t\) for \(i \in \mathcal{X}\) and 
\(a \in \mathcal{A}\).
The states transition according to probability transition
matrices \(\tau \coloneqq \left\{P^a\right\}_{a \in \mathcal{A}}\) which depends only
on their last taken action \(a\). A \emph{policy of the agent} 
is a sequence \(\pi = \left(\mathbf{a_0}, \mathbf{a_1}, \dots \right)\), 
where every \(\mathbf{a_t}(i)\) corresponds to an action in \(\mathcal{A}\) if the system is in 
state \(i\) at time \(t\). For every policy \(\pi\), we have 
a corresponding value function \(v_\pi \in \R^n\), where \(v_\pi(i)\) for a state 
\(i \in \mathcal{X}\) measures the expected cost of that state if the agent were to follow 
policy \(\pi\). This can be expressed by the following recurrence relation 
\begin{align}
 v_\pi(i) \coloneqq c(i, \mathbf{a}_0(i)) + \vartheta \E[j \sim \mathcal{X}]{v_\pi(j)}. 
\end{align}
The goal is to devise algorithms 
to learn an optimal policy \(\pi^*\) that minimizes the expected total cost:

\begin{definition}[Optimal policy]\label{def:optimal-policy}
Given an MDP with state space \(\mathcal{X}\), action space \(\mathcal{A}\) and transition
matrices \(P^a\), let \(\Pi\) be the strategy space of all possibile policies.
Then an optimal policy \(\pi^*\) is one that minimizes the expected total cost,
i.e.,
\(
    \pi^* \coloneqq \arg\min_{\pi \in \Pi} 
    \E{\sum_{t=0}^\infty \vartheta^t c(i_t, \mathbf{a_t}(i_t))}.
\)
\end{definition}
In the robust case we will assume as in \cite{nilim2003robustness, iyengar2005robust}
that the transition matrices \(P^a\) are not fixed and may come from some uncertainty 
region \(\mathcal{P}^a\) and may be chosen adversarially by nature in future 
runs of the model. In this setting, \cite{nilim2003robustness, iyengar2005robust}
prove the following \emph{robust} analogue of the \emph{Bellman recursion}.
A \emph{policy of nature} is a sequence \(\tau \coloneqq \left(\mathbf{P_0}, \mathbf{P_1},
\dots \right)\) where every \(P_t(a) \in \mathcal{P}^a\) 
corresponds to a transition probability matrix chosen from \(\mathcal{P}^a\).
Let \(\mathcal{T}\) denote the set of all such policies of nature.
In other words, a policy \(\tau \in \mathcal{T}\) of nature is a sequence of 
transition matrices that may be played by it in response to the actions of the agent.
For any set \(P \subseteq \R^n\) and vector \(v \in \R^n\), let \(\sigma_P(v) \coloneqq
\sup \left\{ p^\top v \mid p \in P\right\}\) be the \emph{support function}
of the set \(P\). For a state \(i \in \mathcal{X}\), let 
\(\mathcal{P}^a_i\) be the projection onto the \(i^{th}\) row of \(\mathcal{P}^a\).

\begin{theorem}\cite{nilim2003robustness}\label{thm:opt-robust}
We have the following perfect duality relation
\begin{align}
 \min_{\pi \in \Pi} \max_{\tau \in \mathcal{T}} \E[\tau]{\sum_{t=0}^\infty 
 \vartheta^t c\left(i_t, \mathbf{a_t}(i_t)\right)}
 = \max_{\tau \in \mathcal{T}} \min_{\pi \in \Pi} \E[\tau]{\sum_{t=0}^\infty 
 \vartheta^t c\left(i_t, \mathbf{a_t}(i_t)\right)}. 
\end{align}
The optimal value function \(v_{\pi^*}\) corresponding to the optimal policy 
\(\pi^*\) satisfies 
\begin{align}
 v_{\pi^*}(i) = \min_{a \in \mathcal{A}} 
 \left( c(i, a) + \vartheta \sigma_{\mathcal{P}^a_i}(v_{\pi^*})\right),
 \end{align}
 and \(\pi^*\) can then be obtained in a greedy fashion, i.e., \(
 \mathbf{a}^*(i) \in \arg\min_{a \in \mathcal{A}} \left\{c(i, a) + \vartheta 
 \sigma_{\mathcal{P}^a_i}(v)\right\}.
\)
\end{theorem}
The main shortcoming of this approach is that it does not generalize to the \emph{model
free} case where the transition probabilities are not explicitly known but rather
the agent can only sample states according to these probabilities.
In the absence of this knowledge, we cannot compute the support functions of 
the uncertainty sets \(\mathcal{P}^a_i\).
On the other hand
it is often easy to have a \emph{confidence region} \(U^a_i\), e.g.,
a ball or an ellipsoid, corresponding to every 
state-action pair \(i \in \mathcal{X}, a \in \mathcal{A}\)
that quantifies our uncertainty in the simulation, with 
the uncertainty set \(\mathcal{P}^a_i\) being the confidence region 
\(U^a_i\) centered around 
the unknown simulator probabilities. Formally, we define the uncertainty 
sets corresponding to every state action pair in the following fashion.
\begin{definition}[Uncertainty sets]\label{def:uncertainty-set}
 Corresponding to every state-action pair \((i, a)\) we have a \emph{confidence region}
 \(U^a_i\) so that the uncertainty region \(\mathcal{P}^a_i\) 
 of the probability transition matrix 
 corresponding to \((i, a)\) is defined as
 \begin{align}
  \mathcal{P}^a_i \coloneqq \left\{ x + p^a_i \mid x \in U^a_i \right\},
 \end{align}
 where \(p^a_i\) is the \emph{unknown} state transition probability vector from the state 
 \(i \in \mathcal{X}\) to every other state in \(\mathcal{X}\)
 given action \(a\) during the simulation.
\end{definition}
As a simple example, we have the 
ellipsoid \(
U^a_i \coloneqq \left\{x \mid x^\top A^a_i x \le 1, \sum_{i \in \mathcal{X}} x_i = 0\right\}
\)
for some \(n \times n\) psd matrix \(A^a_i\) with the uncertainty set 
\(\mathcal{P}^a_i\) being
\(
 \mathcal{P}^a_i \coloneqq \left\{x + p^a_i \mid x \in U^a_i\right\},
\)
where \(p^a_i\) is the \emph{unknown} simulator state transition probability vector
with which the agent transitioned to a new state during training.
Note that while it may easy to come up with good descriptions of the confidence
region \(U^a_i\), the approach
of \cite{nilim2003robustness, iyengar2005robust} 
breaks down since we have no knowledge of \(p^a_i\) 
and merely observe the new state \(j\) sampled from this distribution. 
See Figure~\ref{fig:tikz:my} for an illustration with the 
confidence regions being an \(\ell_2\) ball of fixed radius \(r\).

In the following sections we develop 
\emph{robust versions} of \(\Q\)-learning, \(\Sarsa\), and  
\(\TD\)-learning which are guaranteed to converge to an approximately 
optimal policy that is robust with respect to this confidence region. 
The robust versions of these iterative algorithms involve an additional 
linear optimization step over the set \(U^a_i\), which in the case of 
\(U^a_i = \left\{\norm{x}_2 \le r\right\}\) simply corresponds to adding fixed noise during
every update. In later sections we will extend it to the function approximation
case where we study linear architectures as well as nonlinear architectures;
in the latter case we derive new stochastic gradient descent algorithms 
for computing approximately robust policies.

\begin{figure}
\begin{center}
  \includegraphics{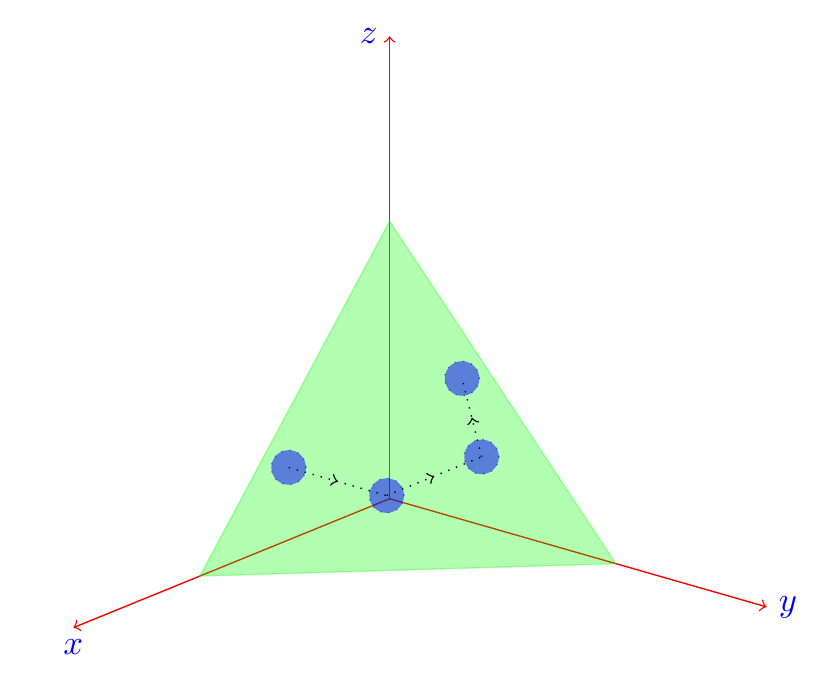}
  \caption{Example transition matrices shown within the probability simplex \(\Delta_n\)
  with uncertainty sets being \(\ell_2\) balls of fixed radius.}
  \label{fig:tikz:my}
  \end{center}
\end{figure}

\section{Robust exact dynamic programming algorithms}\label{sec:exact}
In this section we develop robust versions of exact dynamic programming
algorithms such as \(\Q\)-learning, \(\Sarsa\), and \(\TD\)-learning. These 
methods are suitable for small MDPs where the size \(n\) 
of the state space is not too large.
Note that confidence region \(U^a_i\) must also be constrained to lie within the 
probability simplex \(\Delta_n\), see Figure~\ref{fig:tikz:my}. 
However since we do not have knowledge of the simulator probabilities
\(p^a_i\), we do not know how far away \(p^a_i\) is from the boundary of 
\(\Delta_n\) and so the algorithms will make use of a 
proxy confidence region \(\widehat{U^a_i}\) where we drop the requirement of 
\(\widehat{U^a_i} \subseteq \Delta_n\), to compute the robust optimal policies.
With a suitable choice of step lengths and discount factors we can prove 
convergence to an approximately optimal \(U^a_i\)-robust policy where the 
approximation depends on the difference between the unconstrained proxy
region \(\widehat{U^a_i}\) and the true confidence region \(U^a_i\).
Below we give specific examples of possible choices for simple 
confidence regions.
\begin{enumerate}
 \item \textbf{Ellipsoid:} Let \(\{A^a_i\}_{i, a}\) be a sequence of 
 \(n \times n\) psd matrices. Then we can define 
 the confidence region as
\begin{align}
U^a_i \coloneqq \left\{ x\middle| x^\top A^a_i x \le 1, \sum_{i \in \mathcal{X}} x_i = 0,
 - p^a_{ij} \le x_j \le 1 - p^a_{ij}, \forall j \in \mathcal{X}\right\}.
\end{align}
Note that \(U^a_i\) has some additional linear constraints 
so that the uncertainty set \(\mathcal{P}^a_i \coloneqq \left\{p^a_i + x \mid x \in U^a_i\right\}\)
lies inside \(\Delta_n\). Since we do not know \(p^a_i\), we will make use of the 
proxy confidence region \(\widehat{U^a_i}
\coloneqq \{x \mid x^\top A^a_i x \le 1, \sum_{i \in \mathcal{X}} x_i = 0\}\). 
In particular when \(A^a_i = r^{-1}I_n\) for every 
\(i \in \mathcal{X}, a \in \mathcal{A}\) then this corresponds 
to a spherical confidence interval of \([-r, r]\) in every direction.
In other words, each uncertainty set \(\mathcal{P}^a_i\) is an \(\ell_2\) ball 
of radius \(r\).
\item \textbf{Parallelepiped:} Let \(\{B^a_i\}_{i, a}\)
be a sequence of \(n \times n\) invertible matrices. 
Then we can define the confidence region as
\begin{align}
U^a_i \coloneqq \left\{ x\middle| \norm{B^a_i x}_1 \le 1, \sum_{i \in \mathcal{X}} x_i = 0,
 - p^a_{ij} \le x_j \le 1 - p^a_{ij}, \forall j \in \mathcal{X}\right\}.
\end{align}
As before, we will use the unconstrained parallelepiped
\(\widehat{U^a_i}\) without the \(- p^a_{ij} \le x_j \le 1 - p^a_{ij}\) constraints,
as a proxy for \(U^a_i\) 
since we do not have knowledge \(p^a_i\). In particular if 
\(B^a_i = D\) for a diagonal matrix \(D\), then 
the proxy confidence region \(\widehat{U^a_i}\) corresponds to 
a rectangle. In particular if every diagonal entry is \(r\),  
then every uncertainty set \(\mathcal{P}^a_i\) is an \(\ell_1\)
ball of radius \(r\).
\end{enumerate}
\subsection{Robust \(\Q\)-learning}\label{sec:q-learning}
Let us recall the notion of a \(\Q\)-factor of a state-action pair 
\((i, a)\) and a policy \(\pi\) which in the non-robust setting is defined as 
\begin{align}
 \Q(i, a) \coloneqq c(i, a) + \E[j \sim \mathcal{X}]{v(j)},
\end{align}
where \(v\) is the value function of the policy \(\pi\). In other words,
the \(\Q\)-factor represents the expected cost if we start at state 
\(i\), use the action \(a\) and follow the policy \(\pi\) subsequently.
One may similarly define the \emph{robust} \(\Q\)-factors using a 
similar interpretation and the minimax characterization of 
Theorem~\ref{thm:opt-robust}. 
Let \(\Q^*\) denote the \(\Q\)-factors of the optimal robust policy
and let \(v^* \in \R^n\) be its value function. Note that we may write
the value function in terms of the \(\Q\)-factors as  
\(v^* = \min_{a \in \mathcal{A}} \Q^*(i, a)\).
From Theorem~\ref{thm:opt-robust} we have the following expression for 
\(\Q^*\):
\begin{align}\label{optimal-robust-q-iter}
 \Q^*(i, a) &= c(i, a) + \vartheta\sigma_{\mathcal{P}^a_i}(v^*) \\
 &= c(i, a) + \vartheta \sigma_{U^a_i}(v^*) + \vartheta \sum_{j \in \mathcal{X}} p^a_{ij} 
 \min_{a' \in \mathcal{A}} \Q^*(j, a')\label{eq:unrolling},
\end{align}
where equation~\eqref{eq:unrolling} follows from Definition~\ref{def:uncertainty-set}.
For an estimate \(\Q_t\) of
\(\Q^*\), let \(v_t \in \R^n\) be its value vector, i.e.,
\(v_t(i) \coloneqq \min_{a \in \mathcal{A}} \Q_t(i, a)\). 
The \emph{robust \(\Q\)-iteration} is defined as:
\begin{align}\label{robust-q-iter}
\Q_t(i, a) \coloneqq (1 - \gamma_t) \Q_{t-1}(i, a) + \gamma_t \left(c(i, a) + 
 \vartheta \sigma_{\widehat{U^a_i}}\left(v_{t-1}\right) + \vartheta \min_{a' \in \mathcal{A}} 
 \Q_{t - 1}(j, a') \right),
\end{align}
where a state \(j \in \mathcal{X}\) is sampled with the unknown transition 
probability \(p^a_{ij}\) using the simulator. Note that the robust \(\Q\)-iteration of 
equation~\eqref{robust-q-iter} involves an additional linear optimization 
step to compute the support function \(\sigma_{\widehat{U^a_i}}(v_t)\)
of \(v_t\) over the proxy confidence region \(\widehat{U^a_i}\).
We will prove that iterating equation~\eqref{robust-q-iter} converges to
an approximately optimal policy.
The following definition introduces the notion of an
\(\varepsilon\)-optimal policy, 
see e.g., \cite{bertsekas1995neuro}. The error factor \(\varepsilon\) is also 
referred to as the \emph{amplification factor}. We will treat the \(\Q\)-factors
as a \(\size{\mathcal{X}} \times \size{\mathcal{A}}\) matrix in the definition so that its 
\(\ell_\infty\) norm is defined as usual.
\begin{definition}[\(\varepsilon\)-optimal policy]\label{def:suboptimal}
 A policy \(\pi\) with \(\Q\)-factors \(\Q'\) is \(\varepsilon\)-optimal 
 with respect to the optimal policy \(\pi^*\) with corresponding  
 \(\Q\)-factors \(\Q^*\) if 
 \begin{align}
 \norm{\Q' - \Q^*}_\infty \le \varepsilon \norm{\Q^*}_\infty.
 \end{align}
\end{definition}
The following simple lemma allows us to decompose the 
optimization of a linear function over the proxy uncertainty set
\(\widehat{\mathcal{P}^a_i}\) in terms of linear optimization over 
\(\mathcal{P}^a_i, U^a_i\), and  \(\widehat{U^a_i}\).
\begin{lemma}\label{lem:over-estimate}
 Let \(v \in \R^n\) be any vector and let 
 \(\beta^a_i \coloneqq \max_{y \in \widehat{U^a_i}} \min_{x \in U^a_i} \norm{y - x}_1\).
 Then we have
 \(
  \sigma_{\widehat{\mathcal{P}^a_i}}(v) \le \sigma_{\mathcal{P}^a_i}(v)
  + \beta^a_i \norm{v}_\infty.
 \)
\end{lemma}
\begin{proof}
Note that every point \(p\) in \(\mathcal{P}^a_i\) is of the form 
\(p^a_i + x\) for some \(x \in U^a_i\) and every point \(q \in 
\widehat{\mathcal{P}^a_i}\) is of the form \(p^a_i + y\) for some 
\(y \in \widehat{U^a_i}\), and this correspondence is one to one by 
definition.
For any vector \(v \in \R^n\) and pairs of points
\(p \in \mathcal{P}^a_i\) and \(q \in \widehat{\mathcal{P}^a_i}\) we have
\begin{align}
 q^\top v &= p^\top v + (q - p)^\top v \\
          &\le \sup_{p' \in \mathcal{P}^a_i} (p')^\top v + \left(p^a_i + y - p^a_i - x\right)^\top v \\
          &= \sigma_{\mathcal{P}^a_i}(v) + (y - x)^\top v. \\
          &\le \sigma_{\mathcal{P}^a_i}(v) + (y - x)^\top v\\
& \le \sigma_{\mathcal{P}^a_i}(v) + \left(y^\top v - \min_{x \in U^a_i} x^\top v \right) \\
& \le \sigma_{\mathcal{P}^a_i}(v) + \max_{y \in \widehat{U^a_i}} 
\min_{x \in U^a_i} (y - x)^\top v \\
& \le \sigma_{\mathcal{P}^a_i}(v) + \max_{y \in \widehat{U^a_i}} 
\min_{x \in U^a_i} \norm{y - x}_1 \norm{v}_\infty \\
& \le \sigma_{\mathcal{P}^a_i}(v) + \beta^a_i \norm{v}_\infty. \label{eq:fixed-q}
          \end{align}
Since equation~\eqref{eq:fixed-q} holds for every \(q \in \widehat{\mathcal{P}^a_i}\),
it follows that it also holds for \(\arg \max \sigma_{\widehat{\mathcal{P}^a_i}}(v)\)
so that 
\begin{align}
\sigma_{\widehat{\mathcal{P}^a_i}}(v) \le  \sigma_{\mathcal{P}^a_i}(v) + \beta^a_i \norm{v}_\infty.
\end{align}
\end{proof}
The following theorem proves that under a suitable 
choice of step lengths \(\gamma_t\) and discount factor \(\vartheta\),
the iteration of equation~\eqref{robust-q-iter} converges to an 
\(\varepsilon\)-approximately optimal policy 
with respect to the confidence regions \(U^a_i\). 
\begin{theorem}\label{thm:robust-q-iter}
 Let the step lengths \(\gamma_t\) of the \(\Q\)-iteration algorithm 
 be chosen such that \(\sum_{t = 0}^\infty \gamma_t = \infty\)
 and \(\sum_{t = 0}^\infty \gamma_t^2 < \infty\) and let the discount factor 
 \(\vartheta < 1\). Let \(\beta^a_i\) be as in Lemma~\ref{lem:over-estimate}
 and let \(\beta \coloneqq \max_{i \in \mathcal{X}, a \in \mathcal{A}} \beta^a_i\).
If \(\vartheta (1 + \beta) < 1\) then with probability \(1\) 
the iteration of equation~\eqref{robust-q-iter} 
converges to an \(\varepsilon\)-optimal policy where 
\(\varepsilon \coloneqq 
\frac{\vartheta \beta}{1- \vartheta\left(1 + \beta\right)}.\)
 \end{theorem}
\begin{proof}
 Let \(\widehat{\mathcal{P}^a_i}\) be the proxy uncertainty set for state \(i\in \mathcal{X}\)
 and \(a \in \mathcal{A}\), i.e., 
 \(\widehat{\mathcal{P}^a_i} \coloneqq \left\{ x + p^a_i 
 \mid x \in \widehat{U^a_i}\right\}\). We denote the value function 
 of \(\Q\) by \(v\).
 Let us define the following operator \(H\)
 mapping \(\Q\)-factors to \(\Q\)-factors as follows:
 \begin{align}
(H\Q)(i, a) &\coloneqq c(i, a) + \vartheta \sigma_{\widehat{\mathcal{P}^a_i}}(v).
 \end{align}
 We will first show that a solution \(\Q'\) to the equation \(H\Q = \Q\) 
 is an \(\varepsilon\)-optimal
 policy as in Definition~\prettyref{def:suboptimal}, i.e.,
 \(\norm{Q' - Q^*}_\infty \le \varepsilon \norm{Q^*}_\infty\).
 \begin{align}
  |\Q'(i, a) - \Q^*(i, a)| &= \left|(H\Q')(i, a) - c(i, a) - \vartheta \sigma_{\mathcal{P}^a_i}(v^*)\right|\\
  &= \vartheta \left|\sigma_{\widehat{\mathcal{P}^a_i}}(v')- \sigma_{\mathcal{P}^a_i}(v^*)\right|\\
  &\le \vartheta \left|\max_{y \in \widehat{U^a_i}, x \in U^a_i}
  \norm{y - x}_1\norm{\Q'}_\infty + \sigma_{\mathcal{P}^a_i}(v') - 
  \sigma_{\mathcal{P}^a_i}(v^*)\right| \label{eq:over-estimate}\\
  &\le \vartheta \beta^a_i \norm{\Q'}_\infty + \left|\sigma_{\mathcal{P}^a_i}(v') - 
  \sigma_{\mathcal{P}^a_i}(v^*)\right|\\
  &\le \vartheta \beta \norm{\Q'}_\infty
  + \vartheta \left|\max_{q' \in \mathcal{P}^a_i}\sum_{j \in \mathcal{X}} q'_j\min_{a'' \in \mathcal{A}} 
 \Q'(j, a'') - \max_{q \in \mathcal{P}^a_i}
 \sum_{j \in \mathcal{X}} q_j \min_{a' \in \mathcal{A}} \Q^*(j, a')
 \right|\\
  &\le \vartheta \beta \norm{\Q'}_\infty 
  +\vartheta \left| \max_{q \in \mathcal{P}^a_i}
 \sum_{j \in \mathcal{X}} q_j \left(\min_{a'' \in \mathcal{A}} \Q'(j, a'')
 - \min_{a' \in \mathcal{A}} \Q^*(j, a')\right)\right|\\
 &\le \vartheta \beta \norm{\Q'}_\infty +
 \vartheta \left| \max_{q \in \mathcal{P}^a_i}
 \sum_{j \in \mathcal{X}} q_j \left(\max_{a' \in \mathcal{A}} 
 |\Q'(j, a') - \Q^*(j, a')|\right)\right|\\
 &\le \vartheta \beta \norm{\Q'}_\infty + 
 \vartheta \left| \max_{q \in \mathcal{P}^a_i}
 \sum_{j \in \mathcal{X}} q_j \norm{\Q' - \Q^*}_{\infty}\right| \\
 &\le \vartheta \beta \norm{\Q'}_\infty +  \vartheta \norm{\Q' - \Q^*}_\infty \label{eq:final},
 \end{align}
 where we used Lemma~\ref{lem:over-estimate} to derive equation~\eqref{eq:over-estimate}. 
Equation~\eqref{eq:final} implies that \(\norm{\Q' - \Q^*}_\infty \le
\frac{\vartheta \beta}{1-\vartheta}\norm{\Q'}_\infty\).
If \(\norm{\Q'}_\infty \le \norm{\Q^*}_\infty\) then we are done since \(\frac{\vartheta \beta}{1 - \vartheta} \le
\frac{\vartheta \beta}{1 - \vartheta (1 + \beta)}\). Otherwise assume that \(\norm{\Q'}_\infty > 
\norm{\Q^*}_\infty\) and use the triangle inequality: 
\(\norm{\Q'}_\infty - \norm{\Q^*}_\infty = \left| \norm{\Q'}_\infty - \norm{\Q^*}_\infty\right|
\le \norm{\Q' - \Q^*}_\infty\). This implies that 
\begin{align}
 \frac{1-\vartheta}{\vartheta \beta}\norm{\Q' - \Q^*}_\infty - \norm{\Q^*}_\infty \le \norm{\Q' - \Q^*}_\infty,
\end{align}
from which it follows that \(\norm{\Q' - \Q^*}_\infty \le \varepsilon \norm{\Q^*}_\infty\) under 
the assumption that \(\vartheta (1 + \beta) < 1\) as claimed.
The \(\Q\)-iteration of equation~\eqref{robust-q-iter} can then be reformulated in terms of the 
operator \(H\) as
\begin{align}
 \Q_t(i, a) = (1 - \gamma_t) \Q_{t-1}(i, a) + \gamma_t\left(H \Q_t(i, a) + \eta_t(i, a)\right),
\end{align}
where \(\eta_t(i, a) \coloneqq \min_{a' \in \mathcal{A}} \Q_t(j, a') - \E[j \sim p^a_i]
{\min_{a'\in
\mathcal{A}} \Q_t(j, a')}\) where the expectation is over the states \(j \in \mathcal{X}\)
with the transition probability from state \(i\) to state \(j\) given by \(p^a_j\).
Note that this is an example of a \emph{stochastic approximation algorithm} as in 
\cite{bertsekas1995neuro} with noise parameter \(\eta_t\). Let \(\mathcal{F}_t\)
denote the history of the algorithm until time \(t\).
Note that \(\E[j \sim p^a_i]{\eta_t(i, a)\middle| \mathcal{F}_t} = 0\) by definition and the variance is bounded by
\begin{align}
 \E[j \sim p^a_i]{\eta_t(i, a)^2\middle| \mathcal{F}_t} \le K\left(1 + \max_{\substack{j \in \mathcal{X} \\
 a' \in \mathcal{A}}} \Q_t^2(j, a')\right).
\end{align}
Thus the noise term \(\eta_t\) satisfies the zero conditional mean and bounded variance
assumption (Assumption 4.3 in \cite{bertsekas1995neuro}). Therefore it remains to show that
the operator \(H\) is a \emph{contraction mapping} to argue that iterating 
equation~\eqref{robust-q-iter}
converges to the optimal \(\Q\)-factor \(\Q^*\).
We will show that the operator \(H\) is a contraction mapping with respect to the infinity
norm \(\norm{.}_\infty\). Let \(\Q\) and \(\Q'\) be two different \(\Q\)-vectors with value functions 
\(v\) and \(v'\).
If \(U^a_i\) is not necessarily the same as the unconstrained proxy set \(\widehat{U^a_i}\)
for some \(i \in \mathcal{X}, a \in \mathcal{A}\), 
then we need the discount factor to satisfy \(\vartheta (1 + \beta)\)
in order to ensure convergence.
Intuitively, the discount factor should be small enough that the difference in the estimation due
to the difference of the sets \(U^a_i\) and \(\widehat{U^a_i}\) converges to \(0\) over time.
In this case we show contraction for operator \(H\) as follows
\begin{align}
|(H\Q)(i, a) - (H\Q')(i, a)|
 &\le \vartheta \left| 
 \max_{q \in \widehat{\mathcal{P}^a_i}} \sum_{j \in \mathcal{X}}q_j \left(\min_{a' \in \mathcal{A}} \Q(j, a') - 
 \min_{a''\in \mathcal{A}} \Q'(j, a'')\right)\right|\\
 &\le \vartheta  \max_{q \in \widehat{\mathcal{P}^a_i}} \sum_{j \in \mathcal{X}} q_j \max_{a' \in \mathcal{A}} \left|Q(j, a') - 
 \Q'(j, a')\right|\\
 &\le \vartheta \max_{y \in \widehat{U}, x \in U} \norm{y - x}_1 \norm{\Q- \Q'}_\infty + 
 \vartheta \max_{q \in \mathcal{P}^a_i} \sum_{j \in \mathcal{X}} q_j \norm{\Q - \Q'}_\infty \label{eq:over-estimate2}\\
 &\le \vartheta  \beta \norm{\Q - \Q'}_\infty + 
 \vartheta \norm{\Q - \Q'}_\infty \max_{q \in \mathcal{P}^a_i} \sum_{j \in \mathcal{X}} q_j \\
 &\le \vartheta (\beta  + 1)\norm{\Q - \Q'}_\infty
\end{align}
where we used Lemma~\ref{lem:over-estimate} with vector \(v(j) \coloneqq \max_{a \in \mathcal{A}}
\left|Q(j, a) - Q'(j, a)\right|\) to derive equation~\eqref{eq:over-estimate2} and 
the fact that \(\mathcal{P}^a_i 
\subseteq \Delta_n\) to conclude that
\(\max_{q \in \mathcal{P}^a_i} \sum_{j \in \mathcal{X}} q_j = 1\).
Therefore if \(\vartheta(1 + \beta) < 1\), then it follows that 
the operator \(H\) is a norm contraction and thus the 
robust \(\Q\)-iteration of equation~\eqref{robust-q-iter} converges
to a solution of \(H\Q = \Q\) which is an \(\varepsilon\)-approximately
optimal policy for \(\varepsilon = \frac{\vartheta \beta}{1- \vartheta(1 + \beta)}\),
as was proved before.
\end{proof}

\begin{remark}
If \(\beta = 0\) then note that by Theorem~\ref{thm:robust-q-iter},
the robust \(Q\)-iterations converge to the exact optimal \(\Q\)-factors 
since \(\varepsilon = 0\). Since
\(\beta = \max_{i \in \mathcal{X}, a \in \mathcal{A}}
\frac{\max_{y \in \widehat{U^a_i}}\min_{x \in U^a_i}
 \norm{y - x}_\xi}{\xi_{\min}} \), it follows that \(\beta = 0\) 
 iff \(\widehat{U^a_i} = U^a_i\) for every \(i \in \mathcal{X}, 
 a \in \mathcal{A}\). This happens when the confidence region is small
 enough so that the simplex constraints  
\(- p^a_{ij} \le x_j \le 1 - p^a_{ij} \forall j \in \mathcal{X}\)
in the description of \(\mathcal{P}^a_i\) become redundant for every 
\(i \in \mathcal{X}, a \in \mathcal{A}\). Equivalently every 
\(p^a_i\) is ``far'' from the boundary of the simplex \(\Delta_n\) 
compared to the size of the confidence region \(U^a_i\), 
see e.g., Figure~\ref{fig:tikz:my}.
\end{remark}

\begin{remark}
 Note that simply using the nominal \(\Q\)-iteration without the 
\(\sigma_{\widehat{U^a_i}}(v)\) term does not guarantee
convergence to \(\Q^*\). Indeed, the nominal \(\Q\)-iterations 
converge to \(\Q\)-factors \(\Q'\) where \(\norm{\Q' - \Q^*}_\infty\)
may be arbitrary large. 
This follows easily from observing that 
\(
 |\Q'(i, a) - \Q^*(i, a)| = \left|\sigma_{\widehat{U^a_i}}(v^*)\right|
\),
where \(v^*\) is the value function of \(\Q^*\) 
and so 
\begin{align}\norm{\Q' - \Q^*}_\infty = \max_{ 
i \in \mathcal{X}, a \in \mathcal{A}} \left|\sigma_{\widehat{U^a_i}}(v^*)\right|,
\end{align}
 which can be as high as \(\norm{v^*}_\infty = \norm{Q^*}_\infty\).
See Section~\ref{sec:experiments} for an experimental demonstration
of the difference in the policies learned by the robust and nominal 
algorithms.
\end{remark}

\subsection{Robust \(\Sarsa\)}
Recall that the update rule of \(\Sarsa\) is similar to the
update rule for \(\Q\)-learning except that instead of choosing
the action \(a' = \arg \min_{a' \in \mathcal{A}}Q_{t-1}(j, a')\),
we choose the action \(a''\) where  
with probability \(\delta\), the action \(a''\) is chosen 
uniformly at random from \(\mathcal{A}\) 
and with probability \(1-\delta\), we have
\(a'' = \arg \min_{a' \in \mathcal{A}}Q_{t-1}(j, a')\).
Therefore, it is easy to modify the robust \(\Q\)-iteration of 
equation~\eqref{robust-q-iter} to give us the \emph{robust \(\Sarsa\)}
updates:
\begin{align}\label{eq:robust-sarsa}
\Q_t(i, a) \coloneqq (1 - \gamma_t) \Q_{t-1}(i, a) + \gamma_t \left(c(i, a) + 
 \vartheta \sigma_{\widehat{U^a_i}}\left(v_{t-1}\right) + \vartheta 
 \Q_{t - 1}(j, a'') \right).
\end{align}
In the exact dynamic programming setting, it has the same convergence
guarantees as robust \(\Q\)-learning and can be seen as a corollary of 
Theorem~\ref{thm:robust-q-iter}.
\begin{corollary}
 Let the step lengths \(\gamma_t\) 
 be chosen such that \(\sum_{t = 0}^\infty \gamma_t = \infty\)
 and \(\sum_{t = 0}^\infty \gamma_t^2 < \infty\) and let the discount factor 
 \(\vartheta < 1\). Let \(
  \beta^a_i\) be as in Lemma~\ref{lem:over-estimate}
 and let \(\beta \coloneqq \max_{i \in \mathcal{X}, a \in \mathcal{A}} \beta^a_i\).
If \(\vartheta (1 + \beta) < 1\) then with probability \(1\) 
the iteration of equation~\eqref{eq:robust-sarsa} 
converges to an \(\varepsilon\)-optimal policy where 
\(\varepsilon \coloneqq 
\frac{\vartheta \beta}{1- \vartheta\left(1 + \beta\right)}.\)
 In particular if \(\beta = \beta^a_i = 0\) so that the proxy confidence regions
 \(\widehat{U^a_i}\) are the same as the true confidence regions \(U^a_i\),
 then the iteration~\eqref{eq:robust-sarsa} converges to the true optimum \(\Q^*\).
\end{corollary}

\subsection{Robust \(\TD\)-learning}\label{sec:td-learning}
Recall that \(\TD\)-learning allows us to estimate the value function \(v_\pi\) 
for a given policy \(\pi\). In this section we will generalize the \(\TD\)-learning 
algorithm to the robust case. The main idea behind \(\TD\)-learning 
in the non-robust setting is the following Bellman equation
\begin{align}\label{bellman}
v_\pi(i) \coloneqq \E[j \sim p^{\pi(i)}_i]{c(i, \pi(i)) + v_\pi(j)}.  
\end{align}
Consider a trajectory of the agent
\((i_0, i_1, \dots)\), where \(i_m\) denotes the state of the agent 
at time step \(m\). For a time step \(m\), define the \emph{temporal difference} 
\(d_m\) as 
\begin{align}\label{def:td-vector}
 d_m \coloneqq c(i_m, \pi(i_m)) + \vartheta v_\pi(i_{m+1}) - v_\pi(i_m).
\end{align}
Let \(\lambda \in (0, 1)\). The recurrence relation for \(TD(\lambda)\) 
may be written in terms of the temporal difference \(d_m\) as
\begin{align}\label{eq:v-pi}
 v_\pi(i_k) = \E{\sum_{m = 0}^\infty \left(\vartheta \lambda\right)^{m-k} d_m} + v_\pi(i_k).
\end{align}
The corresponding Robbins-Monro
stochastic approximation algorithm with 
step size \(\gamma_t\) for equation~\eqref{eq:v-pi} is 
\begin{align}\label{eq:td-update}
 v_{t+1}(i_k) \coloneqq v_t(i_k) + \gamma_t\left(\sum_{m = k}^\infty
 \left(\vartheta \lambda\right)^{m-k} d_m\right).
\end{align}
A more general variant of the \(\TD(\lambda)\) iterations uses \emph{eligibility 
coefficients} \(z_m(i)\) for every state \(i \in \mathcal{X}\) and 
temporal difference vector \(d_m\) in the update for equation~\eqref{eq:td-update}
\begin{align}\label{eq:td-general-update}
v_{t+1}(i) \coloneqq v_t(i) + \gamma_t\left(\sum_{m = k}^\infty z_m(i) d_m\right). 
\end{align}
Let \(i_m\) denote the state of the simulator at time step \(m\).
For the discounted case, there are two possibilities for the
eligibility vectors \(z_m(i)\) leading to two different \(\TD(\lambda)\) iterations:
\begin{enumerate}
 \item The \emph{every-visit} \(\TD(\lambda)\) method, where the eligibility coefficients are
 \begin{align*}
  z_m(i) \coloneqq \begin{cases}
                    \vartheta \lambda z_{m-1}(i) \qquad &\text{ if } i_m \neq i \\
                    \vartheta \lambda z_{m-1}(i) + 1 \qquad &\text{ if } i_m = i.
                   \end{cases}
 \end{align*}
 
 \item The \emph{restart} \(\TD(\lambda)\) method, where the eligibility coefficients are
 \begin{align*}
 z_m(i) \coloneqq \begin{cases}
                    \vartheta \lambda z_{m-1}(i) \qquad &\text{ if } i_m \neq i \\
                     1 \qquad &\text{ if } i_m = i.
                   \end{cases} 
 \end{align*}
\end{enumerate}
We make the following assumptions about 
the eligibility coefficients that are sufficient for proof of convergence. 
\begin{assumption}\label{assumption:eligibility}
 The eligibility coefficients \(z_m\) satisfy the following conditions 
 \begin{enumerate}
  \item \(z_m(i) \ge 0\)
  
  \item \(z_{-1}(i) = 0\)
  
  \item \(z_m(i) \le \vartheta z_{m-1}(i)\) if \(i \notin \{i_0, i_1, \dots\}\)
  
  \item The weight \(z_m(i)\) given to the temporal difference \(d_m\) should be 
  chosen before this temporal difference is generated.
 \end{enumerate}
\end{assumption}
Note that the eligibility coefficients of both the every-visit and restart
\(\TD(\lambda)\) iterations satisfy Assumption~\ref{assumption:eligibility}.
In the robust setting, we are interested in estimating the \emph{robust value} of 
a policy \(\pi\), which from Theorem~\ref{thm:opt-robust} we may express as 
\begin{align}\label{robust-bellman}
v_\pi(i) \coloneqq c(i, \pi(i)) + \vartheta\max_{p \in \mathcal{P}^{\pi(i)}_i} \E[j \sim p]
{v_\pi(j)},  
\end{align}
where the expectation is now computed over the probability vector \(p\) chosen 
adversarially from the uncertainty region \(\mathcal{P}^a_i\). 
As in Section~\ref{sec:q-learning}, we may decompose 
\(\max_{p \in \mathcal{P}^a_i} \E[j \sim p]
{v(j)} = \sigma_{\mathcal{P}^a_i}(v)\) as
\begin{align}
\max_{p \in \mathcal{P}^{\pi(i)}_i}\E[j \sim p]{v(j)} = \sigma_{U^{\pi(i)}_i}(v) + 
\E[j \sim p^{\pi(i)}_i]{v(j)},
\end{align}
where \(p^{\pi(i)}_i\) is the transition probability of the agent during a 
simulation. For the remainder of this section, 
we will drop the subscript and just use \(\mathbb{E}\) to denote
expectation with respect to this transition probability \(p^{\pi(i)}_i\).

Define a \emph{simulation} to be a trajectory 
\(\{i_0, i_1, \dots, i_{N_t}\}\) of the agent, 
which is stopped according to a random \emph{stopping time} \(N_t\).
Note that \(N_t\) is a random variable for making stopping decisions
that is not allowed to foresee the future. 
Let \(\mathcal{F}_t\) denote the history of the algorithm
up to the point where the \(t^{th}\) simulation is about to commence. Let \(v_t\)
be the estimate of the value function at the start of the
\(t^{th}\) simulation. Let \(\left\{i_0, i_1, \dots, i_{N_t}\right\}\) be 
the trajectory of the agent during the \(t^{th}\) simulation with \(i_0 = i\).
During training, we generate several simulations of the agent and update
the estimate of the \emph{robust} value function using the 
the \emph{robust temporal difference}
\(\widetilde{d}_m\) which is defined as 
\begin{align}\label{eq:robust-td-vector}
 \widetilde{d}_m &\coloneqq d_m +  \vartheta \sigma_{\widehat{U^{\pi\left(i_m\right)}_{i_m}}}(v_t),\\
 &= c(i_m, \pi(i_m)) 
 + \vartheta v_t(i_{m+1})- v_t(i_m) + \vartheta \sigma_{\widehat{U^{\pi\left(i_m\right)}_{i_m}}}(v_t),
\end{align}
where \(d_m\) is the usual temporal difference defined as before
\begin{align}
 d_m \coloneqq c(i_m, \pi(i_m)) 
 + \vartheta v_t(i_{m+1})- v_t(i_m).
\end{align}
The \emph{robust} \(\TD\)-update is now the usual \(\TD\)-update, except that we 
use the \emph{robust temporal difference} computed over the proxy confidence region:
\begin{align}\label{robust-td-iter}
 v_{t+1}(i) &\coloneqq v_t(i) + \gamma_t\sum_{m = 0}^{N_t-1}z_m(i)\left(\widetilde{d}_m
 \right),\\
 & = v_t(i) + \gamma_t\sum_{m = 0}^{N_t-1}z_m(i)\left(\vartheta 
 \sigma_{\widehat{U^{\pi\left(i_m\right)}_{i_m}}}(v_t) + d_m
 \right).
\end{align}
We define an \(\varepsilon\)-approximate value function for a fixed policy 
\(\pi\) in a way similar to the \(\varepsilon\)-optimal \(\Q\)-factors as 
in Definition~\ref{def:suboptimal}:
\begin{definition}[\(\varepsilon\)-approximate value function]
 Given a policy \(\pi\), we say that a vector \(v' \in \R^n\) is an 
 \(\varepsilon\)-approximation of \(v_\pi\) if the following holds
 \begin{align*}
  \norm{v' - v_\pi}_\infty \le \varepsilon \norm{v_\pi}_\infty.
 \end{align*}
\end{definition}
The following theorem guarantees convergence of the robust \(\TD\)
iteration of equation~\eqref{robust-td-iter} to an 
approximate value function for \(\pi\) under Assumption~\ref{assumption:eligibility}.
\begin{theorem}\label{thm:td-convergence}
Let \(\beta^a_i\) be as in Lemma~\ref{lem:over-estimate}
 and let \(\beta \coloneqq \max_{i \in \mathcal{X}, a \in \mathcal{A}} \beta^a_i\).
 Let \(\rho \coloneqq \max_{i \in \mathcal{X}} \sum_{m = 0}^\infty z_m(i)\). 
 If \(\vartheta (1 + \rho \beta) < 1\) then the robust \(\TD\)-iterations 
 of equation~\eqref{robust-td-iter} converges to an \(\varepsilon\)-approximate value
 function, where 
 \(
  \varepsilon \coloneqq \frac{\vartheta\beta}{1 - \vartheta(1 + \rho \beta)}.
 \)
In particular if \(\beta^a_i = \beta = 0\), i.e., the proxy confidence region 
\(\widehat{U^a_i}\) is the same as the true confidence region \(U^a_i\), then
the convergence is exact, i.e., \(\varepsilon = 0\). Note that in the special case of 
regular \(\TD(\lambda)\) iterations, \(\rho = \frac{\vartheta \lambda}{1- \vartheta \lambda}\).
 \end{theorem}
\begin{proof}
 Let \(\widehat{\mathcal{P}^a_i}\) be the proxy uncertainty set for state \(i\in \mathcal{X}\)
 and action \(a \in \mathcal{A}\) as in the proof of Theorem~\ref{thm:robust-q-iter}, 
 i.e., \(\widehat{\mathcal{P}^a_i} \coloneqq \left\{ x + p^a_i 
 \mid x \in \widehat{U^a_i}\right\}\). 
 Let \(I_t(i) \coloneqq \left\{m\mid i_m = i\right\}\) be the set of time indices 
the \(t^{th}\) simulation visits state \(i\).
We define \(\delta_t(i) \coloneqq \max_{q_m \in \mathcal{P}^{\pi\left(i_m\right)}_{i_m}} 
\E[i_m \sim q_m]{\sum_{m \in I_t(i)} z_m(i) \middle| \mathcal{F}_t}\),
so that we may write the update of equation~\eqref{robust-td-iter} as
 \begin{align}
  v_{t+1}(i) = v_t(i)(1 - \gamma_t\delta_t(i)) + \gamma_t \delta_t(i) 
  \left(\frac{\E{\sum_{m=0}^{N_t-1}z_m(i)
  \widetilde{d}_m
  \middle| \mathcal{F}_t}}{\delta_t(i)}
  + v_t(i)\right) \\ + \gamma_t\delta_t(i)\frac{\vartheta \sum_{m =0}^{N_t-1}z_m(i) \widetilde{d}_m
  - \mathbb{E}\left[\sum_{m=0}^{N_t-1}
  z_m(i)\widetilde{d}_m 
  \middle| \mathcal{F}_t\right]}{\delta_t(i)}.
\end{align}
Let us define the operator \(H_t : \R^n \to \R^n\) corresponding to the \(t^{th}\) simulation as 
\begin{align}
 (H_t v)(i) \coloneqq \frac{ 
 \E{\sum_{m=0}^{N_t-1}z_m(i)
 \left(c(i_m, \pi(i_m)) + \vartheta \sigma_{\widehat{U^{\pi\left(i_m\right)}_{i_m}}}(v) + \vartheta v(i_{m+1}) 
 - v(i_m)\right)
 \middle| \mathcal{F}_t}}{\delta_t(i)} + v(i).
\end{align}
We claim as in the proof of Theorem~\ref{thm:robust-q-iter} that a solution \(v\)
to \(H_t v = v\) must be an \(\varepsilon\)-approximation to 
\(v_\pi\). Define the operator \(H'_t\) with the proxy confidence regions
replaced by the true ones, i.e.,
\begin{align}\label{eq:H_t-def}
 (H'_t v)(i) \coloneqq \frac{
 \mathbb{E}\left[\sum_{m=0}^{N_t-1}z_m(i)
 \left(c(i_m, \pi(i_m)) + \vartheta \sigma_{U^{\pi\left(i_m\right)}_{i_m}}(v) + \vartheta v(i_{m+1}) 
 - v(i_m)\right)
 \middle| \mathcal{F}_t\right]}{\delta_t(i)} + v(i). 
\end{align}
Note that \(H'_t v_\pi = v_\pi\) for the \emph{robust} value function \(v_\pi\) 
since \(c(i_m, \pi(i_m)) + \vartheta \sigma_{U^{\pi\left(i_m\right)}_{i_m}}(v_\pi) + 
\vartheta v_\pi(i_{m+1}) - v_\pi(i_m) = 0\) for every \(i_m \in \mathcal{X}\) by Theorem~\ref{thm:opt-robust}. 
Finally by Lemma~\ref{lem:over-estimate} we have
\begin{align} 
 \sigma_{\widehat{U^{\pi\left(i_m\right)}_{i_m}}}(v)
+ \E{v(i_m)} 
\le \sigma_{U^{\pi\left(i_m\right)}_{i_m}} + \E{v(i_m)} + \beta \norm{v}_\infty,
\end{align}
for any vector \(v\), where the expectation is over the state 
\(i_m \sim p^{\pi\left(i_{m-1}\right)}_{i_{m-1}}\). Thus
for any solution \(v\) to the equation \(H_t v = v\), we have
\begin{align}
 \left|v(i) - v_\pi(i)\right| &= \left| (H_t v)(i) - v_\pi(i)\right|\\
 &\le \left| (H'_t v)(i) - v_\pi(i)\right| + \vartheta \beta \norm{v}_\infty \E{\sum_{m=0}^{N_t - 1} z_m(i)} \\
 &= \left|(H'_t v)(i) - (H'_t v_\pi)(i)\right| + \vartheta \beta \norm{v}_\infty \E{\sum_{m=0}^{N_t - 1} z_m(i)}\\
 &\le \vartheta \norm{v - v_\pi}_\infty + \vartheta \rho \beta\norm{v}_\infty \label{eq:H_t},
 \end{align}
 where equation~\eqref{eq:H_t} follows from equation~\eqref{eq:H_t-def}.
 Therefore the solution to \(H_t v = v\) is an \(\varepsilon\)-approximation to 
\(v_\pi\) for \(\varepsilon = \frac{\vartheta\beta}{1 - \vartheta(1 + \rho \beta)}\)
if \(\vartheta ( 1 + \rho \beta) < 1\) 
as in the proof of Theorem~\ref{thm:robust-q-iter}.
Note that the operator \(H_t\) applied to the iterates \(v_t\) is
\((H_t v_t)(i) = \frac{\mathbb{E}\left[\sum_{m=0}^{N_t-1} z^t_m(i)
\widetilde{d}_{m, t}\middle| 
\mathcal{F}_t\right]}{\delta_t(i)} 
+ v_t(i)\) so that the update of equation~\eqref{robust-td-iter} is a 
\emph{stochastic approximation algorithm} of the form
\begin{align*}
 v_{t+1}(i) = (1 - \widehat{\gamma_t})v_t(i) + \widehat{\gamma_t}\left((H_t v_t)(i) 
 + \eta_t(i)\right),
\end{align*}
where \(\widehat{\gamma_t} =\gamma_t \delta_t(i)\) and
\(\eta_t\) is a noise term with zero mean and is defined as
\begin{align}
\eta_t(i) &\coloneqq \frac{\sum_{m=0}^{N_t-1}z^t_m(i) \widetilde{d}_m -  
\mathbb{E}\left[\sum_{m=0}^{N_t-1}z^t_m(i) \widetilde{d}_m\middle| 
\mathcal{F}_t\right]}{\delta_t(i)}.
\end{align}
Note that by Lemma~5.1 of \cite{bertsekas1995neuro}, the new step sizes
satisfy \(\sum_{t = 0}^\infty \widehat{\gamma_t} = \infty\) and 
\(\sum_{t = 0}^\infty \widehat{\gamma_t}^2 < \infty\) if the original step size
\(\gamma_t\) satisfies the conditions \(\sum_{t=0}^\infty \gamma_t = \infty\) and
\(\sum_{t=0}^\infty \gamma_t^2 < \infty\), since the conditions on the eligibility 
coefficients are unchanged. 
Note that the noise term also satisfies the bounded
variance of Lemma~5.2 of \cite{bertsekas1995neuro} since any 
\(q \in \mathcal{P}^{\pi(i)}_i\) still specifies a
distribution as \(\mathcal{P}^{\pi(i)}_i \subseteq \Delta_n\). 

Therefore, it remains to show that \(H_t\) is a norm contraction with respect to
the \(\ell_\infty\) norm on \(v\). Let us define the operator \(A_t\) as 
\begin{align}
 (A_t v)(i) \coloneqq \frac{ 
\E{\sum_{m=0}^{N_t-1}
z_m(i)\left(\vartheta \sigma_{\widehat{U^{\pi\left(i_m\right)}_{i_m}}}(v) + 
\vartheta v(i_{m+1}\right) - v(i_m) \middle| \mathcal{F}_t}}{\delta_t(i)} + v(i)
\end{align}
and the expression 
\(b_t(i) \coloneqq \frac{\mathbb{E}\left[\sum_{m=0}^{N_t-1} 
c(i_m, \pi(i_m))\middle| \mathcal{F}_t\right]}{\delta_t(i)}\) so that 
\((H_t v)(i) = (A_t v)(i) + b_t(i)\). We will show that 
\(\norm{A_t v}_\infty \le \alpha \norm{v}_\infty\) for some 
\(\alpha < 1\) from which the contraction
on \(H_t\) follows because for any vector \(v'' \in \R^n\) and the 
\(\varepsilon\)-optimal value function \(v' = H_t v'\) we have 
\begin{align}\label{eq:u-hat-u}
 \norm{H_tv'' - v'}_\infty = \norm{H_tv'' - H_tv'}_\infty = \norm{A_t(v'' - v')}_\infty
 \le \alpha \norm{v''-v'}_\infty.
\end{align}
Let us now analyze the expression for \(A_t\). We will show that
\begin{align}
\mathbb{E}\left[\sum_{m=0}^{N_t-1} z_m(i) \left(\vartheta v(i_{m+1}) - v(i_m)
+ \vartheta \sigma_{\widehat{U^{\pi(i)}_i}}(v)\right) + \sum_{m \in I_t(i)}z_m(i)v(i)
\middle| \mathcal{F}_t\right] \le \\ \alpha \norm{v}_\infty 
\mathbb{E}\left[\sum_{m \in I_t(i)}z_m(i)\middle| \mathcal{F}_t\right].
\end{align}
We first replace the \(\sigma_{\widehat{U^{\pi\left(i_m\right)}_{i_m}}}\)
term with \(\sigma_{U^{\pi\left(i_m\right)}_{i_m}}\) using 
Lemma~\ref{lem:over-estimate} while incurring a \(\rho \beta \norm{v}_\infty\) penalty.
Let us collect together the coefficients corresponding to \(v(i_m)\) in the 
expression for the expectation:
\begin{align}\label{eq:exp}
\E{\sum_{m=0}^{N_t-1} z_m(i) 
\left(\vartheta v(i_{m+1}) - v(i_m)
+ \vartheta \sigma_{U^{\pi\left(i_m\right)}_{i_m}}(v)\right) + \sum_{m \in I_t(i)}z_m(i)v(i)
\middle| \mathcal{F}_t} + \vartheta \rho \beta \norm{v}_\infty\\
\le \max_{q_m \in \mathcal{P}^{\pi\left(i_m\right)}_{i_m}}\E[i_m \sim q_m]{\sum_{m=0}^{N_t-1} z_m(i) 
\left(\vartheta v(i_{m+1}) - v(i_m)\right) + \sum_{m \in I_t(i)}z_m(i)v(i)
\middle| \mathcal{F}_t} + \vartheta \rho \beta \norm{v}_\infty \label{eq:subsume} \\
= \max_{q_m \in \mathcal{P}^{\pi\left(i_m\right)}_{i_m}}\E[i_m \sim q_m]{\sum_{m=0}^{N_t} 
(\vartheta z_{m-1}(i) - z_m(i))v(i_m) + \sum_{m \in I_t(i)}z_m(i)v(i)
 \middle| \mathcal{F}_t} + \vartheta \rho \beta \norm{v}_\infty,
\end{align}
where we obtain inequality~\eqref{eq:subsume} by subsuming the 
\(\sigma_{U^{\pi\left(i_m\right)}_{i_m}}\) term
within the expectation since \(\mathcal{P}^{\pi\left(i_m\right)}_{i_m}\) is now part of the 
simplex \(\Delta_n\) and taking the worst possible distribution \(q_m\).
We also used the fact that \(z_{-1}(i) = 0\) and \(z_{N_t}(i) = 0\).
Note that whenever \(i_m \neq i\), the coefficient \(\vartheta z_{m-1}(i) - z_m(i)\) 
of \(v(i_m)\) is nonnegative while whenever \(i_m = i\), then the coefficient
\(\vartheta z_{m-1}(i) - z_{m}(i) + z_{m}(i)\) is also nonnegative.
Therefore, we may bound the right hand side of equation~\eqref{eq:exp} as 
\begin{align}
  \max_{q_m \in \mathcal{P}^{\pi\left(i_m\right)}_{i_m}} \E[i_m \sim q_m]{\sum_{m=0}^{N_t} 
 (\vartheta z_{m-1}(i) - z_m(i))v(i_m) + \sum_{m \in I_t(i)}z_m(i)v(i)
  \middle| \mathcal{F}_t} + \vartheta \rho \beta \norm{v}_\infty \\
 \le \max_{q_m \in \mathcal{P}^{\pi\left(i_m\right)}_{i_m}}\E[i_m \sim q_m]{\sum_{m=0}^{N_t} 
 (\vartheta z_{m-1}(i) - z_m(i)) \norm{v}_\infty +
 \sum_{m \in I_t(i)}z_m(i) \norm{v}_{\infty}
 \middle| \mathcal{F}_t} + \vartheta \rho \beta \norm{v}_\infty. 
\end{align}
Let us now collect the terms
corresponding to a fixed \(z_m(i)\):
\begin{align}
 \max_{q_m \in \mathcal{P}^{\pi\left(i_m\right)}_{i_m}}\E[i_m \sim q_m]
 {\sum_{m=0}^{N_t} (\vartheta z_{m-1}(i) - z_m(i)) \norm{v}_\infty +
 \sum_{m \in I_t(i)}z_m(i) \norm{v}_{\infty}
  \middle| \mathcal{F}_t} + \vartheta \rho \beta\norm{v}_\infty \\
 = \norm{v}_\infty \max_{q_m \in \mathcal{P}^{\pi\left(i_m\right)}_{i_m}}\E[i_m \sim q_m]{
 \sum_{m=0}^{N_t-1} z_m(i) \left(\vartheta  - 1\right) + \sum_{m \in I_t(i)} z_m(i)\middle|
 \mathcal{F}_t} + \vartheta \rho \beta \norm{v}_\infty\\
 \le \norm{v}_\infty \max_{q_m \in \mathcal{P}^{\pi\left(i_m\right)}_{i_m}}\E[i_m \sim q_m]{
 \sum_{m \in I_t(i)} z_m(i) \left(\vartheta  - 1\right) + \sum_{m \in I_t(i)} z_m(i) \middle|
 \mathcal{F}_t} + \vartheta \rho \beta \norm{v}_\infty\label{eq:negative-nu}\\
 \le \norm{v}_\infty \vartheta \left(1 + \rho \beta\right) 
 \E{\sum_{m\in I_t(i)} z_m(i)\middle| \mathcal{F}_t}
\end{align}
where equation~\eqref{eq:negative-nu} follows since \(\vartheta < 1\). 
Therefore setting \(\alpha = \vartheta \left(1 + \rho \beta\right)\),
our claim follows under the assumption that \(\vartheta (1 + \rho \beta) < 1\).
\end{proof}
\section{Robust Reinforcement Learning with function approximation}
In Section~\ref{sec:exact} we derived robust versions of exact dynamic programming
algorithms such as \(\Q\)-learning, \(\Sarsa\), and \(\TD\)-learning 
respectively. If the state space \(\mathcal{X}\) of the MDP is large
then it is prohibitive to maintain a lookup table entry for every state.
A standard approach for large scale MDPs is to use the  
\emph{approximate dynamic programming} (ADP) framework \cite{powell2007approximate}.
In this setting, the problem is parametrized by a smaller 
dimensional vector \(\theta \in \R^{d}\) 
where \(d \ll n = \size{\mathcal{X}}\). 

The natural generalizations of \(\Q\)-learning, \(\Sarsa\), and 
\(\TD\)-learning algorithms of Section~\ref{sec:exact} are 
via the \emph{projected Bellman equation}, where we project back 
to the space spanned by all the parameters in \(\theta \in \R^d\), since 
they are the value functions representable by the model.
Convergence for these algorithms even in the non-robust setting are
known only for linear architectures, see e.g., \cite{bertsekas2011approximate}.
Recent work by \cite{bhatnagar2009convergent} proposed stochastic gradient 
descent algorithms with convergence guarantees for smooth nonlinear function 
architectures, where the problem is framed in terms of minimizing a loss function.
We give robust versions of both these approaches.
\subsection{Robust approximations with linear architectures}\label{sec:robust-apx-rl}
In the approximate setting with linear architectures, 
we approximate the value function \(v_\pi\) of a policy \(\pi\) 
by \(\Phi \theta\) where 
\(\theta \in \R^d\) and \(\Phi\) is an \(n \times d\) \emph{feature matrix}
with rows \(\phi(j)\) for every state \(j \in \mathcal{X}\) representing its 
\emph{feature vector}. Let 
\(S\) be the span of the columns of \(\Phi\), i.e., 
\(S\coloneqq \left\{\Phi \theta \mid \theta \in \R^d\right\}\) is the 
set of representable value functions. 
Define the operator
\(T_\pi : \R^n \to \R^n\)
 as 
 \begin{align}
 (T_\pi v)(i) \coloneqq c(i, \pi(i)) + \vartheta \sum_{j \in \mathcal{X}} p^{\pi(i)}_{ij}v(j),
\end{align}
so that the true value function \(v_\pi\) satisfies \(T_\pi v_\pi = v_\pi\).
A natural approach towards estimating \(v_\pi\) given a current estimate 
\(\Phi\theta_t\) is to compute \(T_\pi \left(\Phi\theta_t\right)\) and 
project it back to \(S\) to get the next parameter \(\theta_{t+1}\). The motivation
behind such an iteration is the fact that the true value function is a fixed point 
of this operation if it belonged to the subspace \(S\).
This gives rise to the \emph{projected Bellman equation}
where the projection \(\Pi\) is typically taken with respect 
to a \emph{weighted Euclidean norm} \(\norm{\cdot}_\xi\),
i.e., \(\norm{x}_\xi = \sum_{i \in \mathcal{X}} \xi_i x_i^2\),
where \(\xi\) is some probability distribution over the states \(\mathcal{X}\),
see \cite{bertsekas2011approximate} for a survey.

In the \emph{model free} case, where we do not have explicit knowledge 
of the transition probabilities, 
various methods like \(\LSTD(\lambda)\), \(\LSPE(\lambda)\), and \(\TD(\lambda)\) 
have been proposed see e.g., \cite{bertsekas1996temporal,
bradtke1996linear, boyan2002technical, nedic2003least, sutton2009convergent, 
sutton2009fast}.
The key idea behind proving convergence for these  
methods is to show that the mapping \(\Pi T_\pi\) is a contraction mapping 
with respect to the \(\norm{\cdot}_\xi\) for some distribution 
\(\xi\) over the states \(\mathcal{X}\). While the operator \(T_\pi\) in the non-robust case is linear
and is a contraction in the \(\ell_\infty\) norm
as in Section~\ref{sec:exact}, 
the projection operator with respect to such norms is not 
guaranteed to be a contraction. However, it is known 
that if \(\xi\) is the steady state distribution 
of the policy \(\pi\) under evaluation, then \(\Pi\)
is non-expansive in \(\norm{\cdot}_\xi\) \cite{bertsekas1995neuro,
bertsekas2011approximate}. Hence because of discounting, the 
mapping \(\Pi T_\pi\) is a contraction.

We generalize these methods to the robust setting via 
the \emph{robust Bellman operators} \(T_{\pi}\) defined as
\begin{align}\label{eq:robust-bellman-operator}
(T_\pi v)(i) \coloneqq c(i, \pi(i)) + \vartheta \sigma_{\mathcal{P}^{\pi(i)}_i}(v).
\end{align}
Since we do not have access to the simulator probabilities 
\(p^a_i\), we will use a proxy set \(\widehat{\mathcal{P}^a_i}\) as in Section~\ref{sec:exact}, 
with the proxy operator denoted by \(\widehat{T_{\pi}}\).
While the iterative methods of the non-robust setting generalize via the 
robust operator \(T_\pi\) and the \emph{robust projected Bellman equation}
\(\Phi \theta = \Pi T_\pi (\Phi\theta)\), it is however not clear how to choose
the distribution \(\xi\) under which the projected operator \(\Pi T_\pi\) 
is a contraction in order to show convergence. Let \(\xi\) be the 
steady state distribution of the 
\emph{exploration policy} \(\widehat{\pi}\) of the MDP with 
transition probability matrix \(P^{\widehat{\pi}}\), i.e. the policy 
with which the agent chooses its actions during the simulation. 
We make the following assumption on the discount factor \(\vartheta\)
as in \cite{tamar2014scaling}.

\begin{assumption}\label{assumption:contraction}
 For every state \(i \in \mathcal{X}\) and action \(a \in \mathcal{A}\), 
 there exists a constant \(\alpha \in (0, 1)\) such that for any
 \(p \in \mathcal{P}^a_i\) we have \(\vartheta  p_j \le \alpha P^{\widehat{\pi}}_{ij}\)
 for every \(j \in \mathcal{X}\).
\end{assumption}

Assumption~\ref{assumption:contraction} might appear artificially restrictive;
however, it is necessary to prove that \(\Pi T_\pi\) is a contraction. 
While \cite{tamar2014scaling} require this assumption for proving convergence
of robust MDPs, a similar assumption is also required in proving convergence of
\emph{off-policy} Reinforcement Learning methods of 
\cite{bertsekas2009projected} where the states are sampled
from an exploration policy \(\widehat{\pi}\) which is not necessarily the 
same as the policy \(\pi\) under evaluation. Note that in the robust setting,
all methods are necessarily \emph{off-policy} since the transition matrices
are not fixed for a given policy.

The following lemma is an \(\xi\)-weighted Euclidean norm version of Lemma~\ref{lem:over-estimate}.
\begin{lemma}\label{lem:over-estimate-xi}
  Let \(v \in \R^n\) be any vector and let 
 \(\beta^a_i \coloneqq \frac{\max_{y \in \widehat{U^a_i}}\min_{x \in U^a_i}
 \norm{y - x}_\xi}{\xi_{\min}}
 \). Then we have
 \begin{align}
  \sigma_{\widehat{\mathcal{P}^a_i}}(v) \le \sigma_{\mathcal{P}^a_i}(v)
  + \beta^a_i \norm{v}_\xi,
 \end{align}
 where \(\xi_{\min} \coloneqq \min_{i \in \mathcal{X}} \xi_i\).
\end{lemma}
\begin{proof}
 Same as Lemma~\ref{lem:over-estimate} except now we take Cauchy-Schwarz with 
 respect to weighted Euclidean norm \(\norm{\cdot}_\xi\) in the following manner
 \begin{align}
  a^\top b \le \frac{a^\top \Xi b}{\xi_{\min}} \le \frac{\norm{a}_\xi \norm{b}_\xi}{\xi_{\min}}. 
 \end{align}
\end{proof}
The following theorem shows that the robust projected Bellman equation is a 
contraction under reasonable assumptions on the discount factor \(\vartheta\).
\begin{theorem}\label{thm:projected-contraction}
 Let \(\beta^a_i\) be as in Lemma~\ref{lem:over-estimate-xi}
 and let \(\beta \coloneqq \max_{i \in \mathcal{X}}\beta^{\pi(i)}_i\). 
 If the discount factor \(\vartheta\) satisfies 
 Assumption~\ref{assumption:contraction} for some \(\alpha\)
 and \(\alpha^2 + \vartheta^2 \beta^2 < \frac{1}{2}\), 
 then the operator \(\widehat{T}_\pi\) is a contraction
 with respect to \(\norm{\cdot}_\xi\). 
 In other words, for any two \(\theta, \theta'\in \R^d\), we have
 \begin{align}
 \norm{\widehat{T_\pi}(\Phi\theta) - \widehat{T_\pi}(\Phi\theta')}^2_\xi \le 
 2\left(\alpha^2 + \vartheta^2 \beta^2\right) \norm{\Phi\theta - \Phi\theta'}^2_\xi 
 < \norm{\Phi\theta - \Phi\theta'}^2_\xi.  
 \end{align}
 If \(\beta_i = \beta = 0\) so that \(\widehat{U^{\pi(i)}_i} = U^{\pi(i)}_i\), 
 then we have a simpler contraction under the assumption that \(\alpha < 1\), i.e.,
 \begin{align}
 \norm{\widehat{T}_\pi(\Phi\theta) - \widehat{T}_\pi(\Phi\theta')}_\xi \le 
 \alpha  \norm{\Phi\theta - \Phi\theta'}_\xi < \norm{\Phi\theta - \Phi\theta'}_\xi.
 \end{align}
 \end{theorem}
 \begin{proof}
Consider two parameters \(\theta\) and \(\theta'\) in \(\R^d\). Then we have 
\begin{align}
 \norm{\widehat{T}_\pi(\Phi^\top \theta) - \widehat{T}_\pi(\Phi^\top \theta')}^2_\xi
 &= \sum_{i \in \mathcal{X}} \xi_i \left(\widehat{T}_\pi(\Phi^\top \theta)(i) - 
 \widehat{T}_\pi(\Phi^\top \theta')(i)\right)^2\\
 &= \vartheta^2 \sum_{i \in \mathcal{X}} \xi_i \left(
 \sigma_{\Phi^\top\left(\widehat{\mathcal{P}^{\pi(i)}_i}\right)}(\theta) - 
 \sigma_{\Phi^\top\left(\widehat{\mathcal{P}^{\pi(i)}_i}\right)}(\theta')\right)^2 \\
 &= \vartheta^2 \sum_{i \in \mathcal{X}} \xi_i  
\left(\sup_{q \in \widehat{\mathcal{P}^{\pi(i)}_i}} q^\top \Phi\theta 
- \sup_{q' \in \widehat{\mathcal{P}^{\pi(i)}_i}} (q')^\top \Phi \theta'\right)^2 \\
&\le \vartheta^2 \sum_{i \in \mathcal{X}} \xi_i 
\left(\sup_{q \in \widehat{\mathcal{P}^{\pi(i)}_i}} q^\top 
\left(\Phi\theta 
-  \Phi \theta'\right)\right)^2\\
&\le \vartheta^2 \sum_{i \in \mathcal{X}}\xi_i \left(\sup_{q \in \mathcal{P}^{\pi(i)}_i} 
\left(q^\top \left(\Phi\theta 
- \Phi \theta'\right)\right) + \beta \norm{\Phi\theta - \Phi\theta'}_\xi \right)^2 \label{eq:overestimate-xi}\\ 
&\le \sum_{i \in \mathcal{X}}\xi_i \left(\alpha\sum_{j \in \mathcal{X}} P^{\widehat{\pi}}_{ij} 
\left(\phi(j)^\top\theta - \phi(j)^\top \theta'\right) + \vartheta \beta \norm{\Phi\theta - \Phi\theta'}_\xi\right)^2
\label{eq:cross}\\
&\le 2\sum_{i \in \mathcal{X}}\xi_i \left(\alpha^2\sum_{j \in \mathcal{X}} P^{\widehat{\pi}}_{ij} 
\left(\phi(j)^\top\theta - \phi(j)^\top \theta'\right)^2 + \vartheta^2 \beta^2 \norm{\Phi\theta - \Phi\theta'}^2_\xi\right)\\
&\le 2(\alpha^2 + \vartheta^2 \beta^2) \norm{\Phi\theta 
- \Phi \theta'}_\xi^2 
\end{align}
where we used Lemma~\ref{lem:over-estimate-xi} and the definition of \(\beta\) in 
line~\eqref{eq:overestimate-xi}, 
the inequality \((a + b)^2 \le 2(a^2 + b^2)\), and the fact that 
\(\left(P^{\widehat{\pi}}_{ij}\right)^2 \le P^{\widehat{\pi}}_{ij}\). Note that 
if \(\beta^{\pi(i)}_i = \beta = 0\) so that the proxy confidence region is the same
as the true confidence region, then we have the 
simple upper bound of 
\(\norm{\widehat{T}_\pi(\Phi^\top \theta) - \widehat{T}_\pi(\Phi^\top \theta')}^2_\xi
\le 
 \alpha^2 \norm{\Phi\theta 
- \Phi \theta'}^2_\xi\) instead of  
\(\norm{\widehat{T}_\pi(\Phi^\top \theta) - \widehat{T}_\pi(\Phi^\top \theta')}^2_\xi
\le 
 2\alpha^2 \norm{\Phi\theta 
- \Phi \theta'}^2_\xi\) since we do not have the cross term in 
equation~\eqref{eq:cross} in this case.
\end{proof}

 The following corollary shows that the solution to the proxy 
projected Bellman equation converges to a solution that is
not too far away from the true value function \(v_\pi\).  
\begin{corollary}\label{cor:projection}
 Let Assumption~\ref{assumption:contraction} hold and let \(\beta\) be as in 
 Theorem~\ref{thm:projected-contraction}. 
 Let \(\widetilde{v}_\pi\) be the fixed point of the projected Bellman equation for 
 the proxy operator \(\widehat{T_\pi}\), i.e., \(\Pi \widehat{T_\pi} \widetilde{v}_\pi = 
 \widetilde{v}_\pi\). Let \(\widehat{v}_\pi\) be the fixed point of the proxy operator 
 \(\widehat{T_\pi}\), i.e., \(\widehat{T_\pi} \widehat{v}_\pi = \widehat{v}_\pi\).
 Let \(v_\pi\) be the true value function of the policy \(\pi\), i.e.,
 \(T_\pi v_\pi = v_\pi\).
 Then the following holds
 \begin{align}
  \norm{\widetilde{v}_\pi - v_\pi}_\xi \le \frac{\vartheta \beta \norm{v_\pi}_\xi + 
  \norm{\Pi v_\pi - v_\pi}_\xi}{1 - \sqrt{2\left(\alpha^2 + \vartheta^2 \beta^2\right)}}.
 \end{align}
In particular if \(\beta_i = \beta = 0\) i.e., the proxy confidence region is actually the true
confidence region, then the proxy projected Bellman equation has a solution satisfying
\(
  \norm{\widetilde{v}_\pi - v_\pi}_\xi \le \frac{ 
  \norm{\Pi v_\pi - v_\pi}_\xi}{1 - \alpha}.
  \)
\end{corollary}
\begin{proof}
 We have the following expression
 \begin{align}
  \norm{\widetilde{v}_\pi - v_\pi}_\xi &\le \norm{\widetilde{v}_\pi - \Pi v_\pi}_\xi + 
  \norm{\Pi v_\pi - v_\pi}_\xi\\
  &\le \norm{\Pi \widehat{T}_\pi \widetilde{v}_\pi - \Pi T_\pi v_\pi}_\xi + 
  \norm{\Pi v_\pi - v_\pi}_\xi\\
 & \le \norm{\Pi \widehat{T}_\pi \widetilde{v}_\pi - \Pi \widehat{T}_\pi v_\pi + 
 \vartheta \beta \norm{v_\pi}_\xi} + 
  \norm{\Pi v_\pi - v_\pi}_\xi \label{eq:widehat-T}\\
 & \le \norm{\Pi \widehat{T}_\pi \widetilde{v}_\pi - \Pi \widehat{T}_\pi v_\pi}_\xi + 
 \vartheta \beta \norm{v_\pi}_\xi + 
  \norm{\Pi v_\pi - v_\pi}_\xi\\
 &\le \sqrt{2(\alpha^2 + \vartheta^2 \beta^2)} \norm{\widetilde{v}_\pi - v_\pi}_\xi + 
 \vartheta \beta \norm{v_\pi}_\xi + 
  \norm{\Pi v_\pi - v_\pi}_\xi,
 \end{align}
where we used Lemma~\ref{lem:over-estimate-xi} to derive inequality~\eqref{eq:widehat-T}
and Theorem~\ref{thm:projected-contraction} to conclude that 
\(\norm{\Pi \widehat{T}_\pi \widetilde{v}_\pi - \Pi \widehat{T}_\pi v_\pi}_\xi
\le \sqrt{2(\alpha^2 + \vartheta^2 \beta^2)}\norm{\widetilde{v}_\pi - v_\pi}_\xi\).
If \(\beta^{\pi(i)}_i = \beta = 0\) so that the proxy confidence regions 
are the same as the true confidence regions, then  
we have \(\alpha\) instead of 
\(\sqrt{2 (\alpha^2 + \vartheta^2 \beta^2)}\) in the last
equation due to Theorem~\ref{thm:projected-contraction}.
\end{proof}
Theorem~\ref{thm:projected-contraction} guarantees that the \emph{robust projected Bellman
iterations} of \(\LSTD(\lambda)\), \(\LSPE(\lambda)\) and \(\TD(\lambda)\)-methods converge, while 
Corollary~\ref{cor:projection} guarantees that the solution it converges to is
not too far away from the true value function \(v_\pi\). We refer the reader to 
\cite{bertsekas2011approximate} for more details on \(\LSTD(\lambda)\), \(\LSPE(\lambda)\)
since their proof of convergence is analogous to that of \(\TD(\lambda)\).
\subsection{Robust stochastic gradient descent algorithms}
While the \(\TD(\lambda)\)-learning algorithms with function approximation with linear
architectures converges to \(v_\pi\) 
if the states are sampled according to the policy \(\pi\), it is known to be 
unstable if the states are sampled in an \emph{off-policy} manner, i.e.,
in the terminology of the previous section \(\widehat{\pi} \neq \pi\).
This issue was addressed by \cite{sutton2009convergent, sutton2009fast}
who proposed a stochastic gradient descent based \(\TD(0)\) algorithm
that converges for linear architectures in the \emph{off-policy} setting. This was further
extended by \cite{bhatnagar2009convergent} who extended it to approximations
using arbitrary smooth functions and proved convergence to a local optimum. 
In this section we show how to extend these off-policy methods to the robust 
setting with uncertain transitions. 
Note that this is an \emph{alternative approach} to the requirement of 
Assumption~\ref{assumption:contraction}, 
since under this assumption all
off-policy methods would also converge.

The main idea of \cite{sutton2009fast} is to devise stochastic gradient
algorithms to minimize the following loss
function called the \emph{mean square projected Bellman error} (\(\operatorname{MSPBE}\))
also studied in \cite{antos2008learning, farahmand2009regularized}.
\begin{align}
 \operatorname{MSPBE}(\theta) \coloneqq \norm{v_\theta - \Pi T_\pi v_\theta}^2_\xi.
\end{align}
Note that the loss function is \(0\) for a \(\theta\) that satisfies 
the \emph{projected Bellman equation}, \(\Phi\theta = T_\pi(\Phi\theta)\).
Consider a linear architecture as in Section~\ref{sec:robust-apx-rl}
where \(v_\theta \coloneqq \Phi \theta\).
Let \(i \in \mathcal{X}\) be a random state chosen with distribution \(\xi_i\). 
Denote \(\phi(i)\) by the shorthand \(\phi\) and \(\phi(i')\) by 
\(\phi'\). Then it is easy to show that
\begin{align}
 \operatorname{MSPBE}(\theta) \coloneqq \norm{v_\theta - \Pi T_\pi v_\theta}^2_\xi
 = \E{d\phi}^\top \E{\phi\phi^\top}^{-1} \E{d\phi},
\end{align}
where the expectation is over the random state \(i\) and 
\(d\) is the temporal difference error for the transition 
\((i, i')\) i.e., \(d\coloneqq c(i, a) +\vartheta \theta^\top \phi' - \theta^\top \phi\),
where the action \(a\) and the new state \(i'\) are chosen according to the
exploration policy \(\widehat{\pi}\). The negative gradient of the \(\operatorname{MSPBE}\)
function is 
\begin{align}\label{eq:mspbe-grad}
 -\frac{1}{2}\nabla \operatorname{MSPBE}(\theta) &= \E{(\phi - \vartheta \phi')\phi^\top}w\\
 &= \E{d\phi} - \vartheta \E{\phi'\phi^\top}w
\end{align}
where \(w = \E{\phi\phi^\top}^{-1}\E{d\phi}\). Both \(d\) and \(w\) depend on
\(\theta\). Since the expectation is hard to compute exactly \cite{sutton2009fast}
introduce a set of weights \(w_k\) whose purpose is to estimate \(w\) for a fixed 
\(\theta\). Let \(d_k\) denote the temporal difference error for a parameter 
\(\theta_k\). The weights \(w_k\) are then updated on a fast time scale as 
\begin{align}\label{eq:update-weight}
 w_{k+1} \coloneqq w_k + \beta_k \left(d_k - \phi_k^\top w_k\right) \phi_k,
\end{align}
while the parameter \(\theta_k\) is updated on a slower timescale in the following 
two possible manners
\begin{align}\label{eq:update-parameter}
 \theta_{k+1} \coloneqq \theta_k + \alpha_k\left(\phi_k - 
 \vartheta \phi'_k\right)(\phi_k^\top w_k) \quad &\text{GTD2}\\
 \theta_{k+1} \coloneqq \theta_k + \alpha_k d_k \phi_k - 
 \vartheta \alpha_k \phi'_k(\phi_k^\top w_k) \quad &\text{TDC}
\end{align}
\cite{bhatnagar2009convergent} extended this to the case of smooth nonlinear 
architectures, where the space \(S \coloneqq \left\{ v_\theta \mid \theta \in \R^d\right\}\)
spanned by all value functions \(v_\theta\) 
is now a differentiable sub-manifold of \(\R^n\) rather than a linear subspace.
Projecting onto such nonlinear manifolds is a computationally hard problem,
and to get around this \cite{bhatnagar2009convergent} project
instead onto the tangent plane at \(\theta\)
assuming the parameter \(\theta\) changes very little in one step.
This allows \cite{bhatnagar2009convergent} to generalize 
the updates of equations~\eqref{eq:update-weight}
and \eqref{eq:update-parameter} with an additional Hessian term \(\nabla^2 v_\theta\)
which vanishes if \(v_\theta\) is linear in \(\theta\).

In the following sections we extend the stochastic gradient algorithms of 
\cite{bhatnagar2009convergent, sutton2009convergent, sutton2009fast}
to the robust setting with uncertain transition matrices. Since the 
number \(n\) of states is prohibitively large, we will make the simplifying 
assumption that \(U^a_i = U\) and \(\widehat{U^a_i} = U^a_i\) for the 
results of the following sections.
\subsubsection{Robust stochastic gradient algorithms with linear architectures}\label{sec:robust-linear-grad}
In this section we extend the results of \cite{sutton2009fast}
to the robust setting, where we are interested in finding a solution to the 
\emph{robust projected Bellman equation} \(\Phi\theta = T_\pi\left(\Phi\theta\right)\),
where \(T_\pi\) is the robust Bellman operator of equation~\eqref{eq:robust-bellman-operator}. 
Let \(\widehat{T}_\pi\)
denote the proxy robust Bellman operators using the proxy uncertainty set
\(\widehat{U}\) instead of \(U\). A natural generalization
of \cite{sutton2009fast} is to introduce the following loss function 
which we call \emph{mean squared robust projected
Bellman error} (MSRPBE): 
\begin{align}
 \operatorname{MSRPBE}(\theta) \coloneqq \norm{v_\theta - \Pi \widehat{T}_\pi v_\theta}^2_\xi,
\end{align}
where the proxy robust Bellman operator \(\widehat{T}\) is used. Note that 
\(\widehat{T}_\pi\) is no longer truly linear in \(\theta\) even for linear 
architectures \(v_\theta = \Phi\theta\) as 
\begin{align}
 (\widehat{T}_\pi \Phi \theta)(i) &= c(i, \pi(i)) + \vartheta \sigma_{\mathcal{P}^{\pi(i)}_i}(\Phi \theta)\\
			&= c(i, \pi(i)) + \vartheta \theta^\top \Phi^\top p^{\pi(i)}_i + 
			\vartheta \sup_{q \in \Phi^\top\left(\widehat{U}\right)} q^\top \theta, 
\end{align}
where \(p^{\pi(i)}_i\) are the simulator transition probability vector. However, under
the assumption that \(\widehat{U}\) is a nicely behaved set such as a ball or an 
ellipsoid, so that changing \(\theta\) in a small neighborhood does not lead to jumps in 
\(\sigma_{\Phi^\top(\widehat{U})}(\theta)\), we may define the gradient 
\(\nabla_\theta \widehat{T}_{\pi}(\Phi\theta)(i)\) as 
\begin{align}
 \nabla_\theta ((\widehat{T}_\pi \Phi \theta)(i)) &\coloneqq \vartheta \Phi^\top p^{\pi(i)}_i +
 \vartheta \arg \max_{q \in \Phi^\top(\widehat{U})} q^\top \theta\\
   &= \vartheta \arg \max_{q \in \Phi^\top\left(\widehat{\mathcal{P}^{\pi(i)}_i}\right)} q^\top \theta.
\end{align}
Recall the \emph{robust temporal difference error} \(\widetilde{d}\) for state \(i\) 
with respect to the proxy set \(\widehat{U}\) as in equation~\eqref{eq:robust-td-vector}
\begin{align}
 \widetilde{d} \coloneqq c(i, \pi(i)) + \vartheta v_{\theta}(i') + \sigma_{\widehat{U}}
 (v_\theta) - v_\theta(i).
\end{align}
Under the assumption that \(\E{\phi\phi^\top}\)
is full rank, we may write the 
\(\operatorname{MSRPBE}\) loss function
in terms of the robust temporal difference errors \(\widetilde{d}\) of
equation~\eqref{eq:robust-td-vector} as in \cite{sutton2009fast}:
\begin{align}\label{eq:msrpbe-td}
 \operatorname{MSRPBE}(\theta) = \E{\widetilde{d}\phi}^\top \E{\phi\phi^\top}^{-1}
 \E{\widetilde{d}\phi}.
\end{align}
Note that if \(\E{\phi\phi^\top}\) is full rank, then \(\operatorname{MSRPBE}(\theta) = 0\) 
if and only if \(\E{\widetilde{d} \phi} = 0\) because of equation~\eqref{eq:msrpbe-td}.
Define 
\begin{align}
 \mu_{P}(\theta) \coloneqq \nabla \max_{y \in P} y^\top v_\theta = 
\nabla \max_{y \in P} y^\top \Phi \theta = 
\Phi^\top \arg\max_{y \in P}y^\top \theta =  \arg\max_{y \in \Phi^\top(P)} y^\top \theta
\end{align} 
for any convex compact set \(P \subset \R^n\), so 
that the gradient of the \(\operatorname{MSRPBE}\) loss function can be written as 
\begin{align}
 -\frac{1}{2}\nabla \operatorname{MSRPBE}(\theta) &= 
 \E{\left(\phi - \vartheta \mu_{\widehat{U}}(\theta) - \vartheta \phi'\right)\phi^\top} \E{\phi\phi^\top}^{-1}
 \E{\widetilde{d}\phi},\\
 &= \E{\left(\phi - \vartheta \mu_{\widehat{U}}(\theta)\right)\phi^\top}w,\\
 &= \E{\widetilde{d}\phi} - \vartheta \E{\phi' \phi^\top}w -
 \vartheta \E{\mu_{\widehat{U}}(\theta)\phi^\top}w
\end{align}
where \(w = \E{\phi\phi^\top}^{-1}\E{\widetilde{d}\phi}\) is the same as in
equation~\eqref{eq:mspbe-grad} and \cite{sutton2009fast}. Therefore, as in 
\cite{sutton2009fast} we have an estimator \(w_k\) for the weights \(w\)
for a fixed parameter \(\theta_k\) as 
\begin{align}\label{eq:robust-w-update}
 w_{k+1} \coloneqq w_k + \beta_k \left(\widetilde{d}_k - \phi^\top_k w_k\right)\phi_k,
\end{align}
with the corresponding parameter \(\theta_k\) being updated as
\begin{align}\label{eq:robust-theta-update}
 \theta_{k+1} \coloneqq \theta_k + \alpha_k\left(\phi_k - 
 \vartheta \mu_{\widehat{U}}(\theta) - \phi'_k\right)(\phi_k^\top w_k) \quad &\text{robust-GTD2}\\
 \theta_{k+1} \coloneqq \theta_k + \alpha_k \widetilde{d}_k \phi_k - 
 \vartheta  \alpha_k (\phi'_k +  \mu_{\widehat{U}}(\theta))(\phi_k^\top w_k) \quad &\text{robust-TDC}.
\end{align}
\textbf{Run time analysis:}
Let \(T_n(P)\) denote the time to optimize linear functions over the 
convex set \(P\) for some \(P \subset \R^n\). 
Note that the values \(v_\theta(i)\) can be computed simply in \(O(d)\) 
time. Thus the updates of \emph{robust-GTD2} and \emph{robust-TDC} can be computed in 
\(O\left(d + T_n\left(\widehat{U}\right)\right)\) time.
In particular if the set \(\widehat{U}\) is a simple set like an 
ellipsoid with associated matrix \(A\), then the optimum value 
\(\sigma_{\widehat{U}}(v_\theta)\)
is simply \(\sqrt{\theta^\top \Phi^\top A \Phi\theta}\), 
where \(\Phi\) is the feature matrix. 
In this case we only need to compute
\(\Phi^\top A\Phi\) once and store it for future use.
However, note that this still takes time polynomial in \(n\), 
which is undesirable for \(n \gg d\). In this case, we need to 
to make the assumption that there are good rank-\(d\) approximations to 
\(\widehat{U}\) i.e., \(A \approx B B^\top\) for some \(n \times d\)
matrix \(B\). 

Thus the total run time for each update in this case is \(O(d^2)\).
If the uncertainty set is spherically symmetric, i.e., a ball, then the expression
is simply \(\norm{\Phi\theta}_2\) and the robust temporal difference errors of 
equation~\eqref{eq:robust-td-vector} and the updates of equation~\eqref{eq:robust-w-update}
and \eqref{eq:robust-theta-update} can be viewed simply as regular updates of 
\cite{sutton2009convergent} with an added \emph{noise term}.
\subsubsection{Robust stochastic gradient algorithms with nonlinear architectures}


In this section we generalize the results of Section~\ref{sec:robust-linear-grad} where 
we show how to extend the algorithms of equation~\eqref{eq:robust-w-update}
and \eqref{eq:robust-theta-update} to the case when the value function \(v_\theta\)
is no longer a linear function of \(\theta\). This also generalizes the results
of \cite{bhatnagar2009convergent} to the robust setting with corresponding 
robust analogues of \emph{nonlinear GTD2} and \emph{nonlinear TDC} respectively.
Let \(\mathcal{M} \coloneqq \left\{ v_\theta \mid \theta \in \R^d\right\}\) be the 
manifold spanned by all possible value functions and let \(P\mathcal{M}_\theta\)
be the \emph{tangent plane} of \(\mathcal{M}\) at \(\theta\). Let \(T\mathcal{M}_\theta\) 
be the \emph{tangent space}, i.e., the translation of \(P\mathcal{M}_\theta\) to the 
origin. In other words, \(T\mathcal{M}_\theta \coloneqq \left\{\Phi_\theta u \mid 
u \in \R^d\right\}\), where \(\Phi_\theta\) is an \(n \times d\) matrix with 
entries \(\Phi_\theta(i, j) \coloneqq \frac{\partial}{\partial \theta_j} v_\theta(i)\).
Let \(\Pi_\theta\) denote the projection with to the weighted Euclidean norm 
\(\norm{\cdot}_\xi\) on to the space \(T\mathcal{M}_\theta\), so that
\begin{align}
 \Pi_\theta = \Phi_\theta \left(\Phi_\theta \Xi \Phi_\theta\right)^{-1}\Phi_\theta^\top \Xi
\end{align}
where \(\Xi\) is the \(n\times n\) diagonal matrix with entries \(\xi_i\) for 
\(i \in \mathcal{X}\) as in Section~\ref{sec:robust-apx-rl}. The \emph{mean squared
projected Bellman equation} (\(\operatorname{MSPBE}\)) loss 
function considered by \cite{bhatnagar2009convergent} can then be defined as
\begin{align}
 \operatorname{MSPBE}(\theta) = \norm{v_\theta - \Pi_\theta Tv_\theta}_\xi^2,
\end{align}
where we now project to the the tangent space \(T\mathcal{M}_\theta\).
The robust version of the \(\operatorname{MSPBE}\) loss function, the 
\emph{mean squared robust projected Bellman equation} (\(\operatorname{MSRPBE}\)) loss can 
then be defined in terms of the \emph{robust Bellman operator} over the 
proxy uncertainty set \(\widehat{U}\)
\begin{align}
 \operatorname{MSRPBE}(\theta) = \norm{v_\theta - \Pi_\theta \widehat{T}v_\theta}_\xi^2,
\end{align}
and under the assumption that \(\E{\nabla v_\theta(i) \nabla v_\theta(i)^\top}\)
is non-singular, this may be expressed in terms of the
\emph{robust temporal difference} error 
\(\widetilde{d}\) of equation~\eqref{eq:robust-td-vector}
as in \cite{bhatnagar2009convergent} and equation~\eqref{eq:msrpbe-td}: 
\begin{align}\label{eq:msrpbe-nonlinear}
\operatorname{MSRPBE}(\theta) = \E{\widetilde{d}\nabla v_\theta(i)}^\top 
\E{\nabla v_\theta(i) \nabla v_\theta(i)^\top}^{-1}
 \E{\widetilde{d}\nabla v_\theta(i)}, 
\end{align}
where the expectation is over the states \(i \in \mathcal{X}\) drawn from the 
distribution \(\xi\). Note that under the assumption that 
\(\E{\nabla v_\theta(i) \nabla v_\theta(i)^\top}\) is non-singular, 
it follows due to equation~\eqref{eq:msrpbe-nonlinear} 
that \(\operatorname{MSRPBE}(\theta) = 0\) 
if and only if \(\E{\widetilde{d}\nabla v_\theta(i)} = 0\).
Since \(v_\theta\) is no longer linear in \(\theta\), we need to redefine
the gradient \(\mu\) of \(\sigma\) for any convex, compact set \(P\) as 
\begin{align}
 \mu_{P}(\theta) \coloneqq \nabla \max_{y \in P} y^\top v_\theta
 = \Phi_\theta^\top \arg\max_{y \in P} y^\top v_\theta,
\end{align}
where \(\Phi_\theta(i) \coloneqq \nabla v_\theta(i)\).
The following lemma expresses the gradient \(\nabla \operatorname{MSRPBE}(\theta)\)
in terms of the \emph{robust temporal difference errors}, see Theorem~1
of \cite{bhatnagar2009convergent} for the non-robust version.

\begin{lemma}\label{lem:robust-non-linear-grad}
 Assume that \(v_{\theta}(i)\) is twice differentiable with respect to 
 \(\theta\) for any \(i \in \mathcal{X}\) and that \(W(\theta) \coloneqq 
 \E{\nabla v_\theta(i) \nabla v_\theta(i)^\top}\) is non-singular in a
 neighborhood of \(\theta\). Let \(\phi \coloneqq 
 \nabla v_\theta(i)\) and define for any \(u \in \R^d\)
 \begin{align}
  h(\theta, u) \coloneqq - \E{(\widetilde{d} - \phi^\top u)\nabla^2 v_\theta(i)u}.
 \end{align}
Then the gradient of \(\operatorname{MSRPBE}\) with respect to \(\theta\) can 
be expressed as 
\begin{align}
 - \frac{1}{2}\nabla\operatorname{MSRPBE}(\theta) = 
 \E{\left(\phi - \vartheta \mu_{\widehat{U}}(\theta)  - \vartheta \phi'\right)\phi^\top}w + h(\theta, w),
\end{align}
where \(w = \E{\phi\phi^\top}^{-1}\E{\widetilde{d}\phi}\) as before.
\end{lemma}

\begin{proof}
 The proof is similar to Theorem~1 of \cite{bhatnagar2009convergent} by using
 \(\mu_{\widehat{U}}(\theta)\) as the gradient of \(\sigma_{\widehat{U}}(\theta)\).
\end{proof}
Lemma~\ref{lem:robust-non-linear-grad} leads us to the following robust 
analogues of \emph{nonlinear GTD} and \emph{nonlinear TDC}. The update of the
weight estimators \(w_k\) is the same as in equation~\eqref{eq:robust-w-update}
\begin{align}\label{eq:robust-nonlinear-w-update}
 w_{k+1} \coloneqq w_k + \beta_k \left(\widetilde{d}_k - \phi^\top_k w_k\right)\phi_k,
\end{align}
with the parameters \(\theta_k\) being updated on a slower timescale as 
\begin{align}\label{eq:robust-nonlinear-theta-update}
 \theta_{k+1} &\coloneqq \Gamma\left(\theta_k + \alpha_k\left\{ \left(\phi_k  
 - \vartheta \phi'_k - 
 \vartheta \mu_{\widehat{U}}(\theta)\right)(\phi_k^\top w_k) - h_k\right\}\right) \quad &\text{robust-nonlinear-GTD2}\\
 \theta_{k+1} &\coloneqq \Gamma\left(\theta_k + \alpha_k \left\{ \widetilde{d}_k \phi_k  
 - \vartheta \phi'_k - \vartheta  \mu_{\widehat{U}}(\theta)(\phi_k^\top w_k) - h_k\right\}\right)\quad 
 &\text{robust-nonlinear-TDC},
\end{align}
where \(h_k \coloneqq \left(\widetilde{d}_k - \phi_k^\top w_k\right)
\nabla^2 v_{\theta_k}\left(i_k\right)w_k\) and \(\Gamma\) is a projection
into an appropriately chosen compact set \(C\) with a smooth boundary 
as in \cite{bhatnagar2009convergent}. As in \cite{bhatnagar2009convergent} the
main aim of the projection is to prevent the parameters to diverge in the 
early stages of the algorithm due to the nonlinearities in the algorithm.
In practice, if \(C\) is large enough that it contains the set of all 
possible solutions \(\left\{\theta \middle| \E{\widetilde{d}\nabla v_\theta(i)} = 0\right\}\)
then it is quite likely that no projections will happen. However, we require
the projection for the convergence analysis of the 
\emph{robust-nonlinear-GTD2} and \emph{robust-nonlinear-TDC} algorithms,
see Section~\ref{sec:gradient-convergence}.
Let \(T_n(P)\) denote the time to optimize a linear function over the set 
\(P \subset \R^n\). Then the run time is 
\(O\left(d + T_n\left(\widehat{U}\right)\right)\). 
If \(\widehat{U}\) is an ellipsoid with associated matrix 
\(A\), then an approximate optimum may be computed by sampling,
if we have a rank-\(d\) approximation to \(A\), i.e.,
\(A \approx BB^\top\) for some \(n \times d\) matrix.
If \(\widehat{U}\) is spherically symmetric, then 
the \(\sigma\left(\widehat{U}\right)\) is simply \(\norm{v_\theta}_2\)
so that the updates of equations~\eqref{eq:robust-nonlinear-w-update}
and \eqref{eq:robust-theta-update} may be viewed as the regular
updates of \cite{bhatnagar2009convergent} with an added noise term.

\subsubsection{Convergence analysis}\label{sec:gradient-convergence}
In this section we provide a convergence analysis for the \emph{robust-nonlinear-GTD2}
and \emph{robust-nonlinear-TDC} algorithms of equations~\eqref{eq:robust-nonlinear-w-update}
and \eqref{eq:robust-nonlinear-theta-update}. Note that this also proves convergence of the 
\emph{robust-GTD2} and \emph{robust-TDC} algorithms of equations~\eqref{eq:robust-w-update}
and \eqref{eq:robust-theta-update} as a special case. Given the set \(C\) let 
\(\mathcal{C}(C)\) denote the space of all \(C \to \R^d\) continuous functions.
Define as in \cite{bhatnagar2009convergent} the function \(\widehat{\Gamma}: \mathcal{C}(C)
\to \mathcal{C}\left(\R^d\right)\) 
\begin{align}
 \widehat{\Gamma}f(\theta) \coloneqq \lim_{\varepsilon \to 0} 
 \frac{\Gamma(\theta + \varepsilon f(\theta)) - \theta}{\varepsilon}. 
\end{align}
Since \(\Gamma(\theta) = \arg \min_{\theta' \in C} \norm{\theta - \theta'}\) and 
the boundary of \(C\) is smooth, it follows that \(\widehat{\Gamma}\) is well defined.
Let \(\mathring{C}\) denote the interior of \(C\) and \(\partial{C}\) denote its
boundary so that \(\mathring{C} = C \setminus \partial{C}\). If \(\theta \in 
\mathring{C}\), then \(\widehat{\Gamma}v(\theta) = v(\theta)\), otherwise
\(\widehat{\Gamma}(\theta)\) is the projection of \(v(\theta)\) to the tangent
space of \(\partial{C}\) at \(\theta\). Consider the following ODE as in 
\cite{bhatnagar2009convergent}:
\begin{align}\label{eq:ode}
 \dot{\theta} = \widehat{\Gamma}\left(-\frac{1}{2}\nabla\operatorname{MSRPBE}\right)(\theta), \quad 
 \theta(0) \in C
\end{align}
and let \(K\) be the set of all stable equilibria of equation~\eqref{eq:ode}. 
Note that the solution set \(\left\{\theta \middle| \E{\widetilde{d} \phi} = 0
\right\} \subset K\).
The following theorem shows that under the assumption of Lipschitz continuous gradients
and suitable assumptions on the step lengths \(\alpha_k\) and \(\beta_k\) 
and the uncertainty set \(\widehat{U}\),
the updates of equations~\eqref{eq:robust-nonlinear-w-update}
and \eqref{eq:robust-nonlinear-theta-update} converge.
\begin{theorem}[Convergence of \emph{robust-nonlinear-GTD2}]\label{thm:convergence}
 Consider the robust nonlinear updates of equations~\eqref{eq:robust-nonlinear-w-update}
and \eqref{eq:robust-nonlinear-theta-update} with step sizes that satisfy
\(\sum^\infty_{k=0} \alpha_k = \sum^\infty_{k=0} \beta_k = \infty\),
\(\sum^\infty_{k=0} \alpha_k^2, \sum^\infty_{k=0} \beta^2_k < \infty\),
and \(\frac{\alpha_k}{\beta_k} \to 0\) as \(k \to \infty\).
Assume that for every \(\theta\) we have \(\E{\phi_\theta \phi_\theta^\top}\)
is non-singular. Also assume that the matrix \(\Phi_\theta\) of gradients of the value
function defined as \(\Phi_\theta(i) \coloneqq \nabla v_\theta(i)\)
is Lipschitz continuous with constant \(L\), i.e.,
\(\norm{\Phi_\theta - \Phi_{\theta'}}_2 \le L \norm{\theta - \theta'}_2\).
Then with probability \(1\), \(\theta_k \to K\) as \(k \to \infty\).
\end{theorem}

\begin{proof}
The argument is similar to the proof of Theorem~2 in \cite{bhatnagar2009convergent}. 
The only thing we need to verify is the Lipschitz continuity of the robust version 
\(\widetilde{g}(\theta_k, w_k)\) of
the function \(g(\theta_k, w_k)\) of \cite{bhatnagar2009convergent} defined as 
\begin{align}
 \widetilde{g}(\theta_k, w_k) \coloneqq \E{(\phi_k - \vartheta \mu_{\widehat{U}}(\theta)\phi_k^\top w_k - h_k
 \mid \theta_k, w_k},
\end{align}
where \(g(\theta_k, w_k)\) is defined as \(g(\theta_k, w_k) \coloneqq
\E{(\phi_k - \vartheta \phi'_k(\theta)\phi_k^\top w_k - h_k
 \mid \theta_k, w_k}\), where \(\phi'_k\) is the features of the state \(i'\) 
 the simulator transitions to from state \(i\). Thus we only need to verify
 Lipschitz continuity of \(\mu_{\widehat{U}}(\theta)\). 
 Let \(y^* \coloneqq \arg\max_{y \in \widehat{U}} y^\top v_\theta\)
 and let \(z^* \coloneqq \arg\max_{z \in \widehat{U}} z^\top v_\theta'\).
 \begin{align}
  \norm{\mu_{\widehat{U}}(\theta) - \mu_{\widehat{U}}(\theta')}_2
  &=  \norm{\Phi_\theta^\top y^* -
  \Phi_{\theta'}^\top z^*}_2\\
  &\le \norm{\Phi_\theta^\top y^* -
  \Phi_{\theta'}^\top y^*}_2\\
  &\le \norm{\Phi_\theta - \Phi_{\theta'}}_2 \norm{y^*}_2\\
  &\le \norm{\Phi_\theta - \Phi_{\theta'}}_2 \arg\max_{y \in \widehat{U}} \norm{y}_2\\
  &\le \left(L \arg\max_{y \in \widehat{U}} \norm{y}_2 \right)\norm{\theta - \theta'}_2.
 \end{align}
 Therefore the \(\mu_{\widehat{U}}(\theta)\) is Lipschitz continuous with 
 constant \(L \arg\max_{y \in \widehat{U}} \norm{y}_2\).
\end{proof}

\begin{corollary}
 Under the same conditions as in Theorem~\ref{thm:convergence}, the \emph{robust-GTD2},
 \emph{robust-TDC} and \emph{robust-nonlinear-TDC} algorithms satisfy with probability \(1\) 
 that \(\theta_k \to K\) as \(k \to \infty\).
\end{corollary}

\section{Experiments}\label{sec:experiments}
We implemented robust versions of \(\Q\)-learning, \(\Sarsa\), and
\(\TD(\lambda)\)-learning as described in Section~\ref{sec:exact} and evaluated
their performance against the nominal algorithms using the OpenAI gym
framework \cite{brockman2016openai}. The environments considered for
the exact dynamic programming algorithms are the text environments of
\textbf{FrozenLake-v0}, \textbf{FrozenLake8x8-v0}, \textbf{Taxi-v2},
\textbf{Roulette-v0}, \textbf{NChain-v0}, as well as the control tasks
of \textbf{CartPole-v0}, \textbf{CartPole-v1},
\textbf{InvertedPendulum-v1}, together with the continuous control
tasks of \textbf{MuJoCo} \cite{todorov2012mujoco}. To test the performance of
the robust algorithms, we perturb the models slightly by choosing with
a small probability \(p\) a random state after every action. The size
of the confidence region \(U^a_i\) for the robust model is chosen by a
\(10\)-fold cross validation using line search. After the \(\Q\)-table
or the value functions are learned for the robust and the nominal
algorithms, we evaluate their performance on the true environment.  To
compare the true algorithms we compare both the \emph{cumulative
  reward} as well as the \emph{tail distribution function} (complementary
cumulative distribution function) as in \cite{tamar2014scaling} which
for every \(a\) plots the probability that the algorithm earned a
reward of at least \(a\).

\begin{figure}
\centering
\begin{minipage}{.5\textwidth}
  \centering
  \includegraphics[scale=0.38]{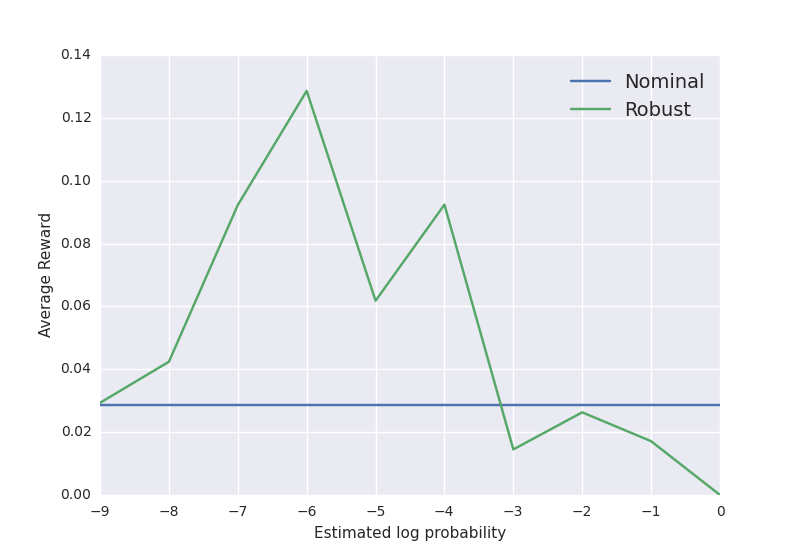}
\end{minipage}%
\begin{minipage}{.5\textwidth}
  \centering
  \includegraphics[scale=0.38]{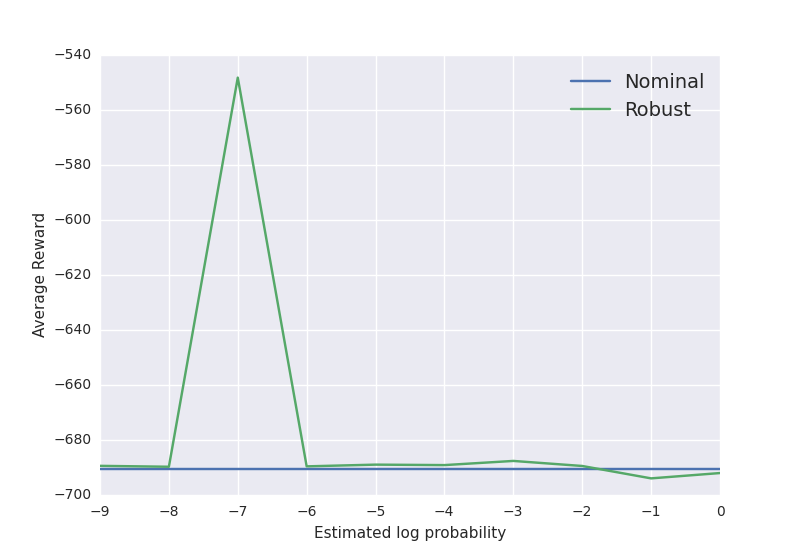}
  \end{minipage}
  \caption{Performance of robust models with different sizes of 
  confidence regions on two environments. Left: \textbf{FrozenLake-v0}
Right: \textbf{Acrobot-v1}}
  \label{fig:cv-frozen-v0}
\end{figure}
Note that there is a tradeoff in the 
performance of the robust algorithms versus the nominal algorithms
in terms of the value \(p\). As the value of \(p\) increases,
we expect the robust algorithm to gain an edge over the 
nominal ones as long as \(\widehat{U}\) is still within the 
simplex \(\Delta_n\). Once we exceed the simplex \(\Delta_n\) however,
the robust algorithms decays in performance. This is due to the presence of the 
\(\beta\) term in the convergence results, which
is defined as 
\begin{align}
 \beta \coloneqq \max_{i \in \mathcal{X}, a \in \mathcal{A}} 
 \max_{y \in \widehat{U^a_i}} \min_{x \in U^a_i} \norm{y - x}_1,
\end{align}
and it grows larger proportional to how much the proxy confidence
region \(\widehat{U}\) is outside \(\Delta_n\).  Note that while
\(\beta\) is \(0\), the robust algorithms converge to the exact
\(\Q\)-factor and value function, while the nominal algorithm does
not.  However, since large values of \(\beta\) also lead to suboptimal
convergence, we also expect poor performance for too large confidence
regions, i.e., large values of \(p\).  Figure~\ref{fig:cv-frozen-v0}
depicts how the size of the confidence region affects the
performance of the robust models; note that the. Note that
the average score appears somewhat erratic as a function 
of the size of the uncertainty set, 
however this is due to our small sample size used in the line search.
 See Figures~\ref{fig:frozen8x8-1e-2}, \ref{fig:frozen8x8-1e-1},
\ref{fig:frozen-1e-2}, \ref{fig:cartpole-v0-1e-3},
\ref{fig:cartpole-v0-1e-2}, \ref{fig:cartpole-v0-3e-1},
\ref{fig:cartpolev1-1e-1}, \ref{fig:cartpole-v1-3e-1},
\ref{fig:taxiv2-1e-1}, and \ref{fig:invertedpendulum1-0.1}
 for a comparison of the best robust model and the nominal model.

\begin{figure}
\centering
\begin{minipage}{.32\textwidth}
  \centering
  \includegraphics[scale=0.23]{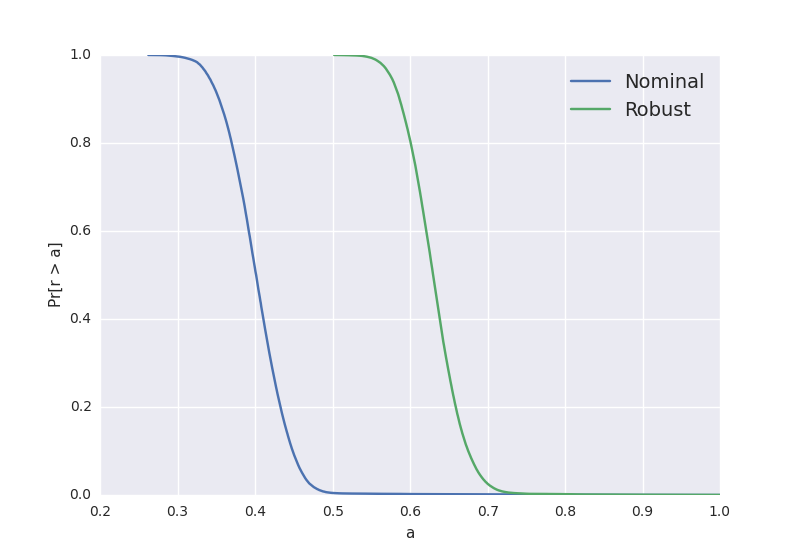}
\end{minipage}%
\begin{minipage}{.32\textwidth}
  \centering
  \includegraphics[scale=0.23]{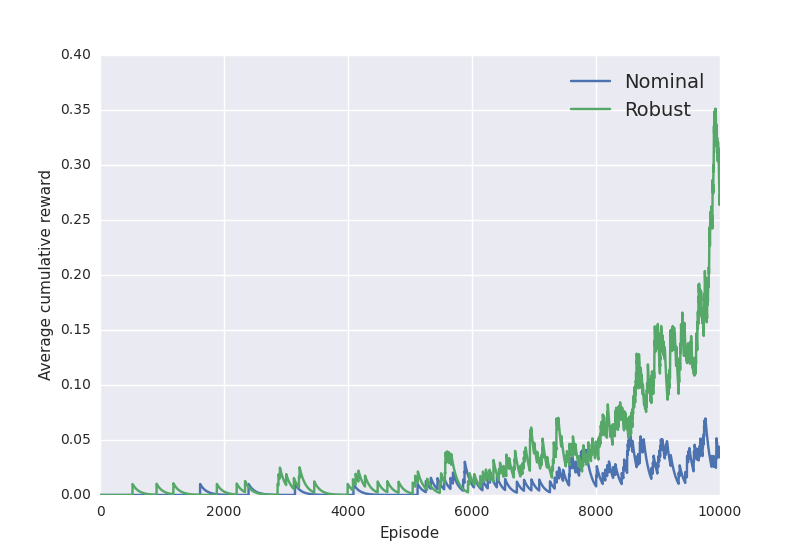}
\end{minipage}%
\begin{minipage}{.32\textwidth}
  \centering
  \includegraphics[scale=0.23]{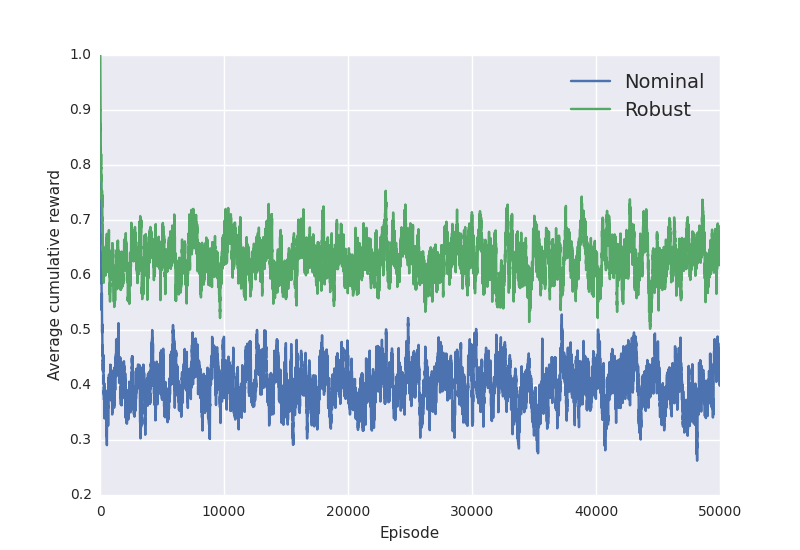}
  \end{minipage}
  \caption{Tail distribution and cumulative rewards during transient and
  stationary phase of robust vs nominal \(\Q\)-learning
  on \textbf{FrozenLake8x8-v0} with \(p = 0.01\).}
  \label{fig:frozen8x8-1e-2}
\end{figure}

\begin{figure}
\centering
\begin{minipage}{.32\textwidth}
  \centering
  \includegraphics[scale=0.23]{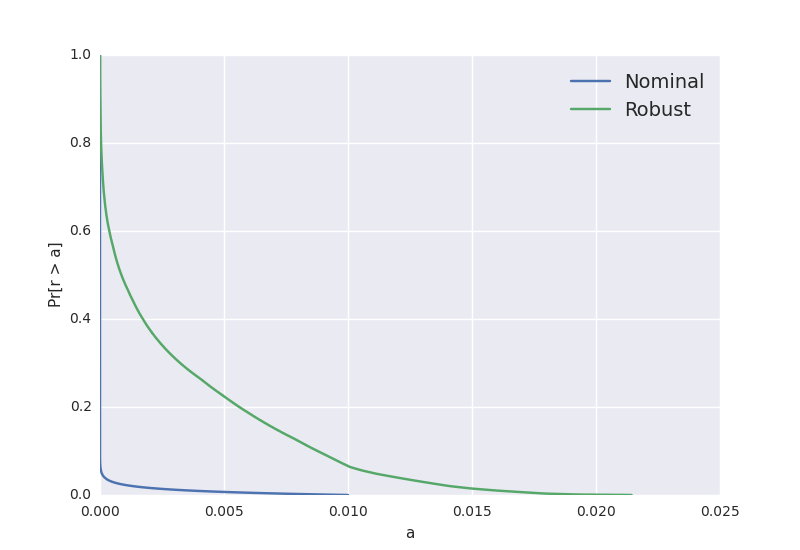}
\end{minipage}%
\begin{minipage}{.32\textwidth}
  \centering
  \includegraphics[scale=0.23]{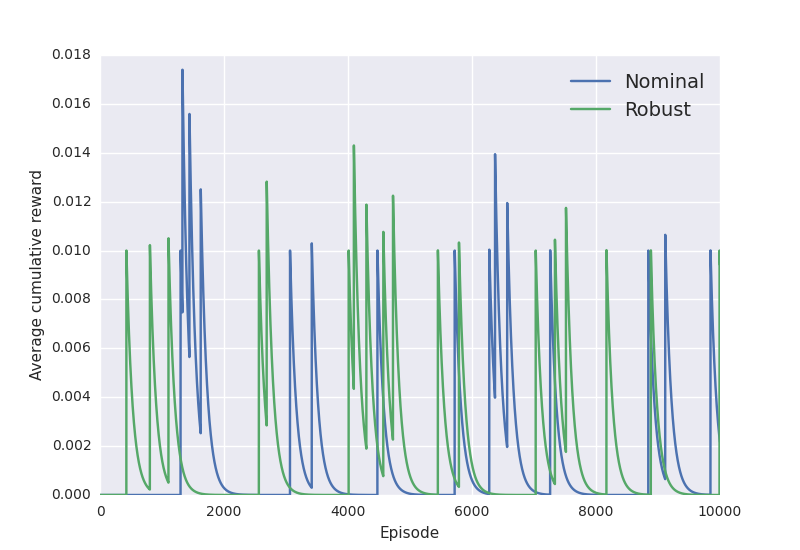}
\end{minipage}%
\begin{minipage}{.32\textwidth}
  \centering
  \includegraphics[scale=0.23]{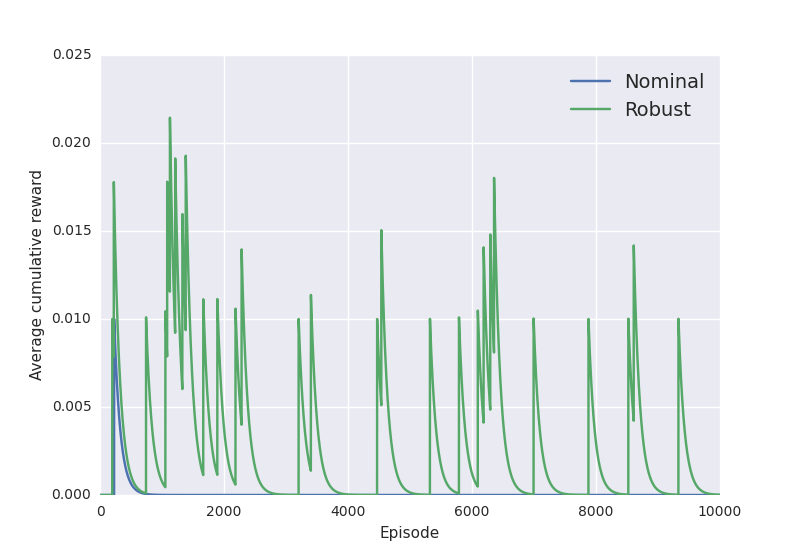}
  \end{minipage}
  \caption{Tail distribution and cumulative rewards during transient and
  stationary phase of robust vs nominal \(\Q\)-learning
  on \textbf{FrozenLake8x8-v0} with \(p = 0.1\).}
  \label{fig:frozen8x8-1e-1}
\end{figure}

\begin{figure}
\centering
\begin{minipage}{.32\textwidth}
  \centering
  \includegraphics[scale=0.23]{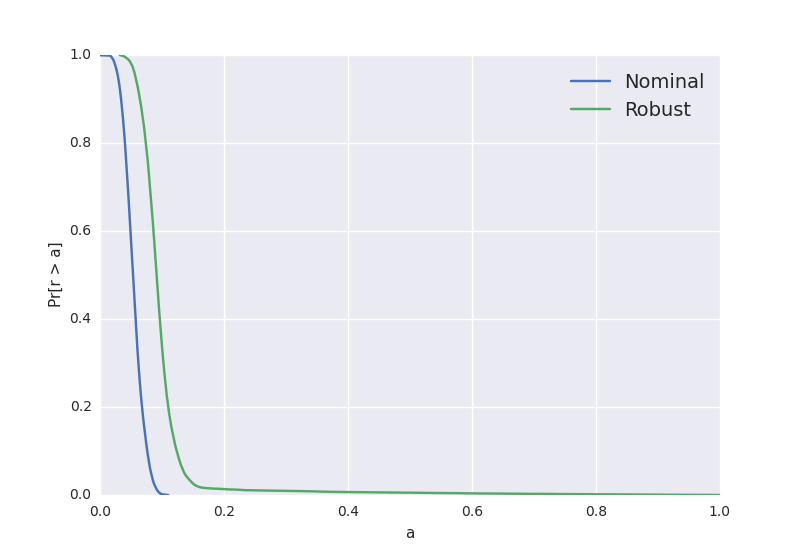}
\end{minipage}%
\begin{minipage}{.32\textwidth}
  \centering
  \includegraphics[scale=0.23]{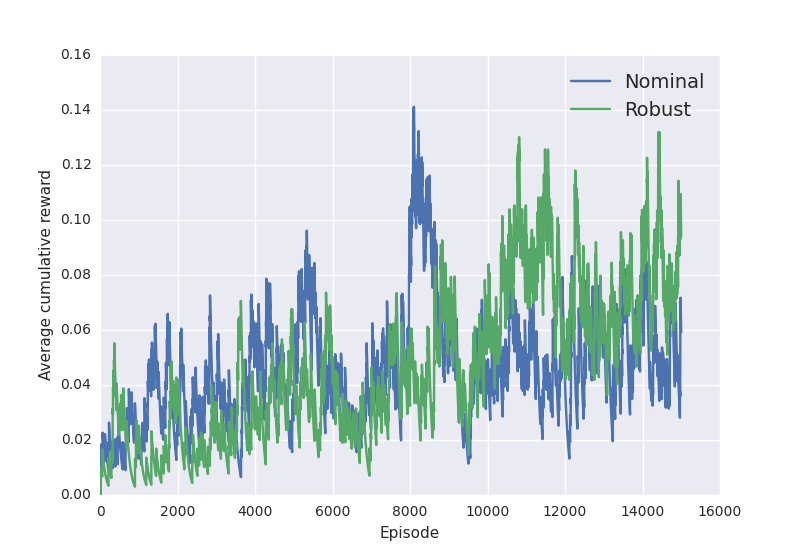}
\end{minipage}%
\begin{minipage}{.32\textwidth}
  \centering
  \includegraphics[scale=0.23]{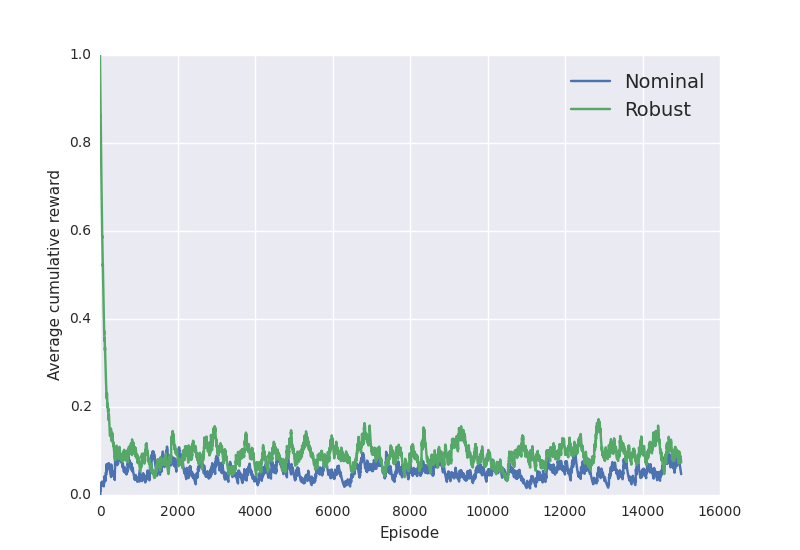}
  \end{minipage}
  \caption{Tail distribution and cumulative rewards during transient and
  stationary phase of robust vs nominal \(\Q\)-learning
  on \textbf{FrozenLake-v0} with \(p = 0.1\).}
  \label{fig:frozen-1e-2}
\end{figure}

\begin{figure}
\centering
\begin{minipage}{.32\textwidth}
  \centering
  \includegraphics[scale=0.23]{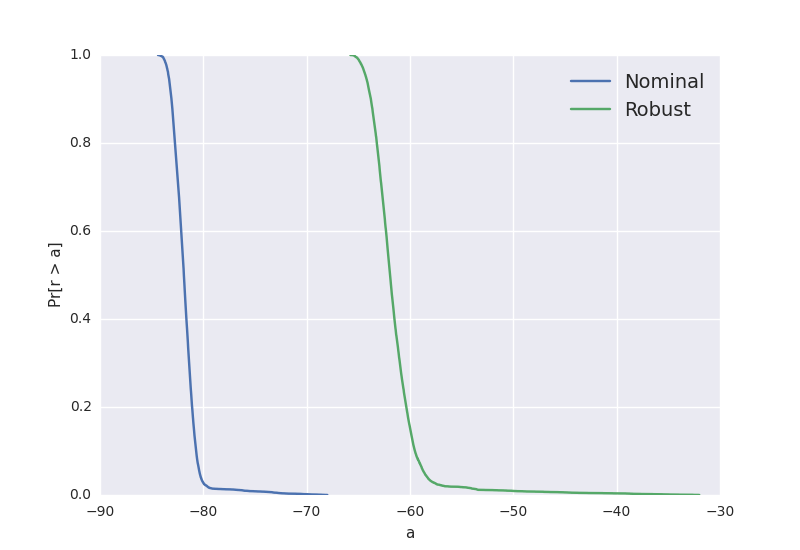}
\end{minipage}%
\begin{minipage}{.32\textwidth}
  \centering
  \includegraphics[scale=0.23]{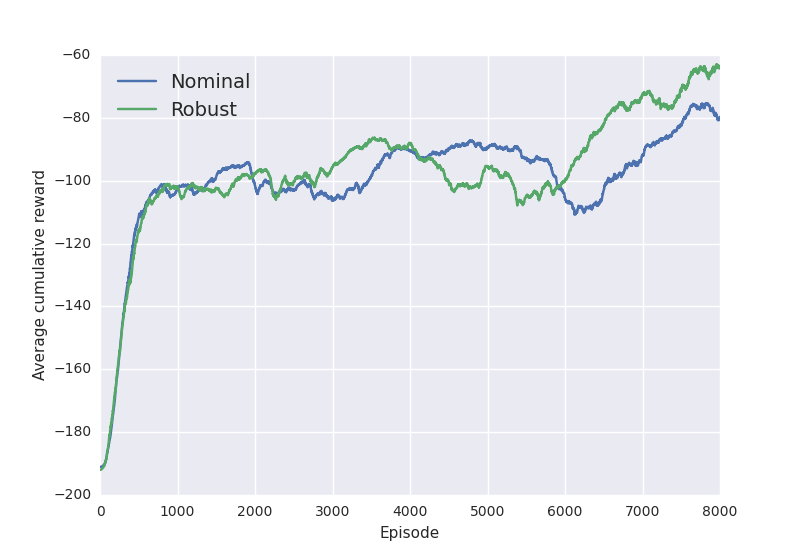}
\end{minipage}%
\begin{minipage}{.32\textwidth}
  \centering
  \includegraphics[scale=0.23]{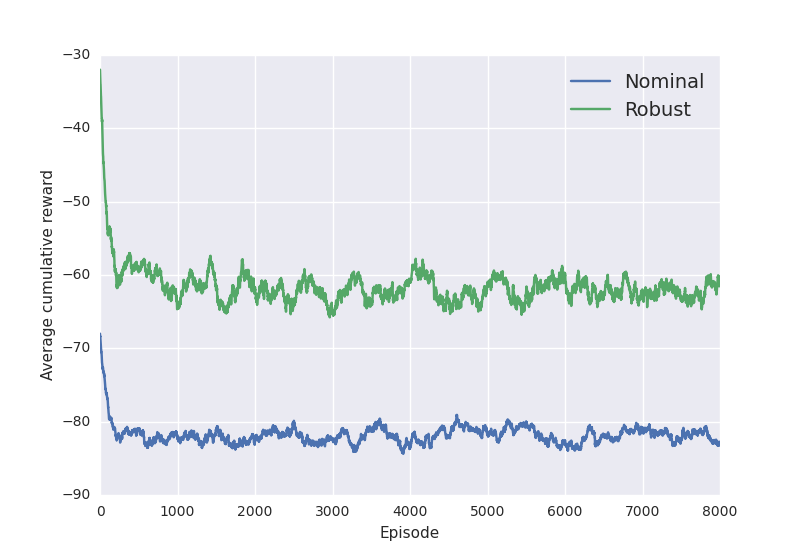}
  \end{minipage}
  \caption{Tail distribution and cumulative rewards during transient and
  stationary phase of robust vs nominal \(\Q\)-learning
  on \textbf{CartPole-v0} with \(p = 0.001\).}
  \label{fig:cartpole-v0-1e-3}
\end{figure}

\begin{figure}
\centering
\begin{minipage}{.32\textwidth}
  \centering
  \includegraphics[scale=0.23]{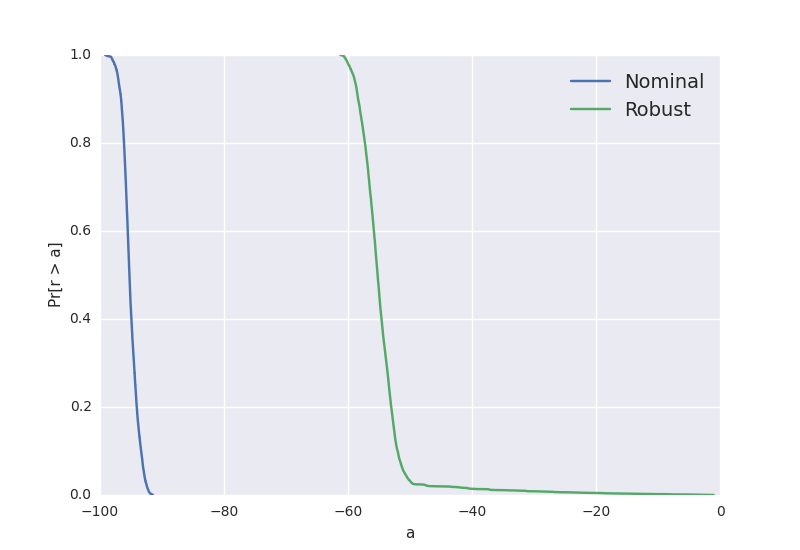}
\end{minipage}%
\begin{minipage}{.32\textwidth}
  \centering
  \includegraphics[scale=0.23]{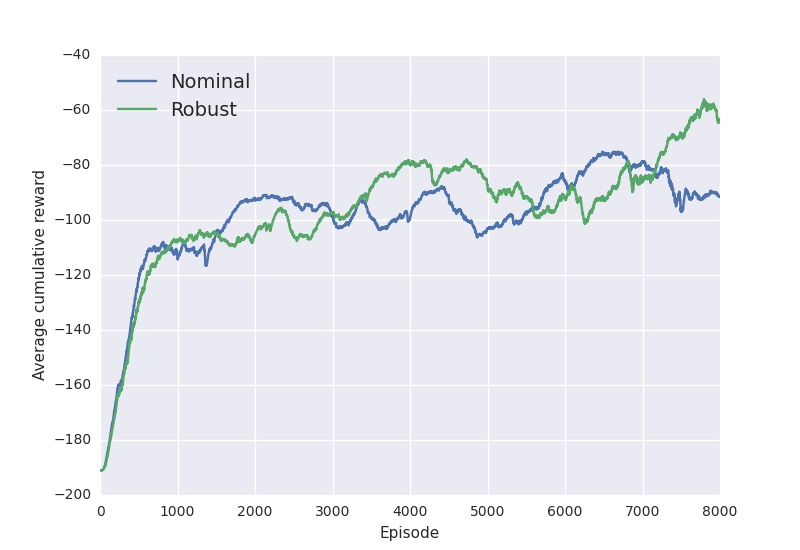}
\end{minipage}%
\begin{minipage}{.32\textwidth}
  \centering
  \includegraphics[scale=0.23]{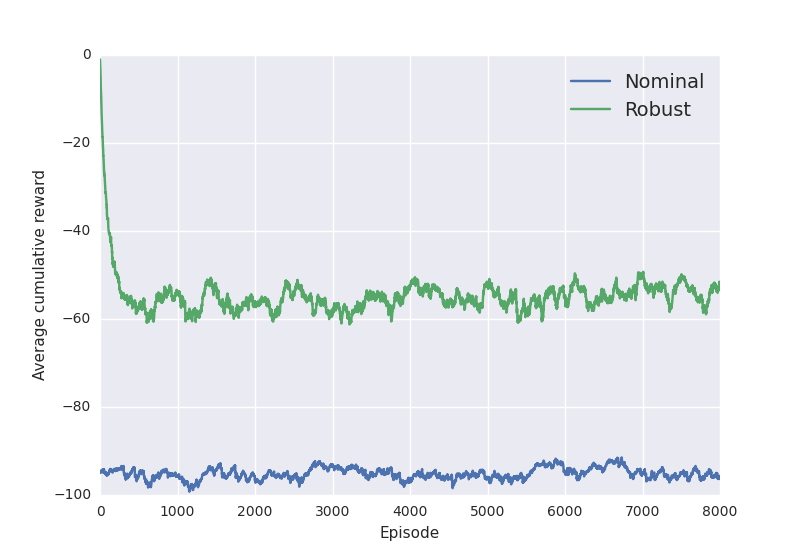}
  \end{minipage}
  \caption{Tail distribution and cumulative rewards during transient and
  stationary phase of robust vs nominal \(\Q\)-learning
  on \textbf{CartPole-v0} with \(p = 0.01\).}
  \label{fig:cartpole-v0-1e-2}
\end{figure}

\begin{figure}
\centering
\begin{minipage}{.32\textwidth}
  \centering
  \includegraphics[scale=0.23]{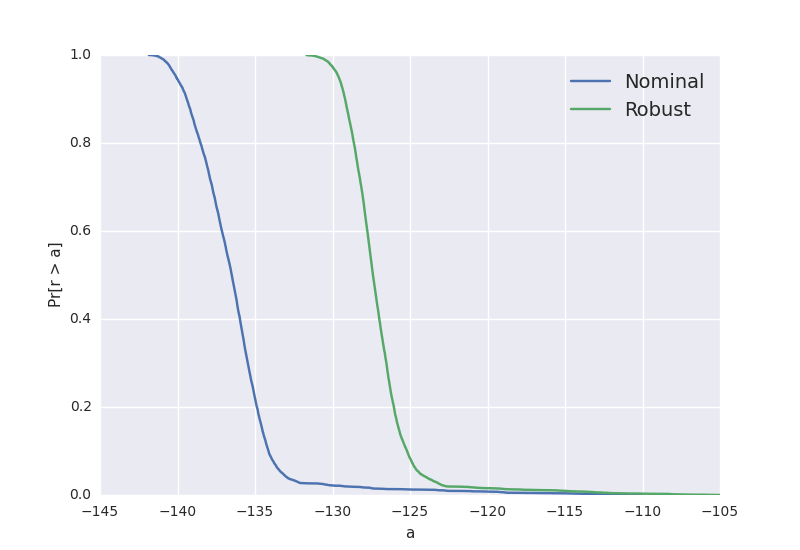}
\end{minipage}%
\begin{minipage}{.32\textwidth}
  \centering
  \includegraphics[scale=0.23]{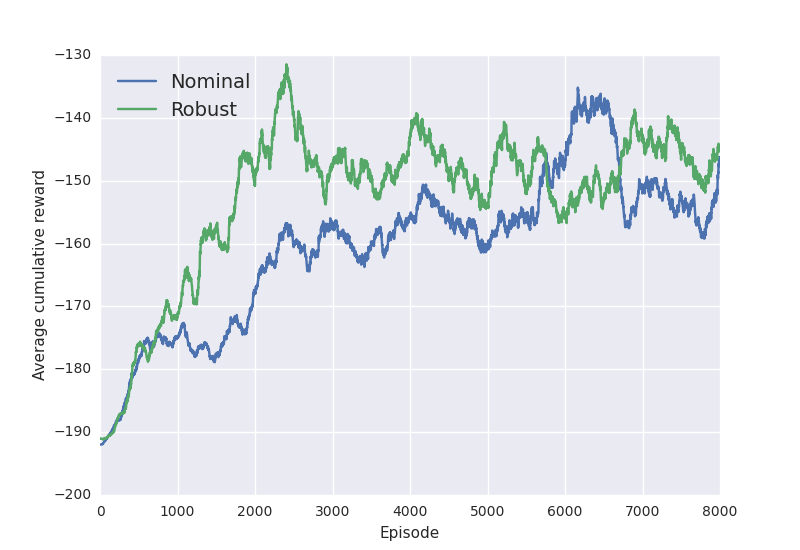}
\end{minipage}%
\begin{minipage}{.32\textwidth}
  \centering
  \includegraphics[scale=0.23]{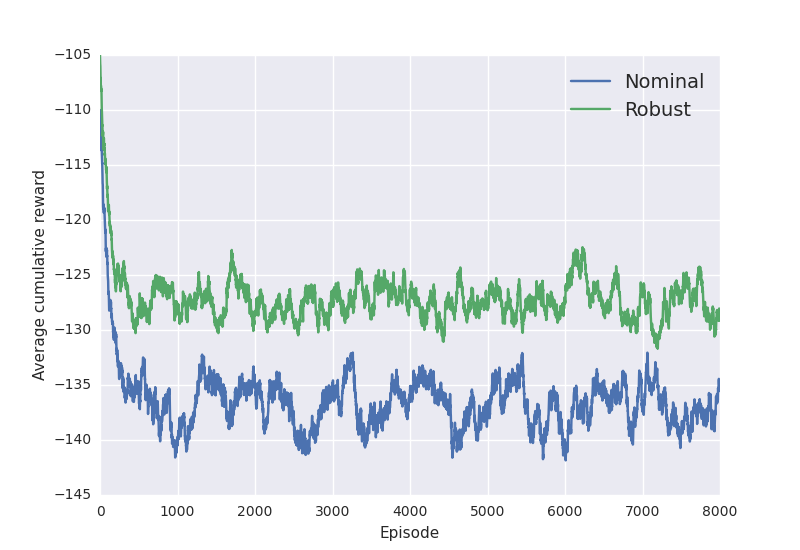}
  \end{minipage}
  \caption{Tail distribution and cumulative rewards during transient and
  stationary phase of robust vs nominal \(\Q\)-learning
  on \textbf{CartPole-v0} with \(p = 0.3\).}
  \label{fig:cartpole-v0-3e-1}
\end{figure}

\begin{figure}
\centering
\begin{minipage}{.32\textwidth}
  \centering
  \includegraphics[scale=0.23]{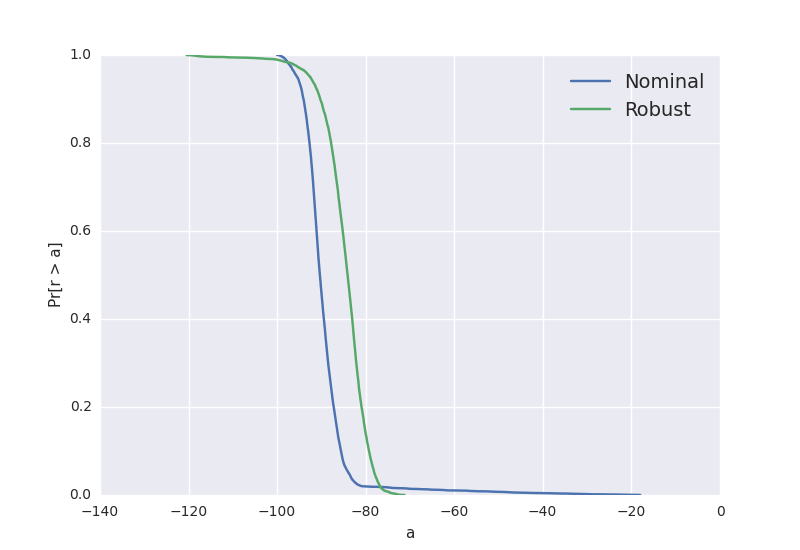}
\end{minipage}%
\begin{minipage}{.32\textwidth}
  \centering
  \includegraphics[scale=0.23]{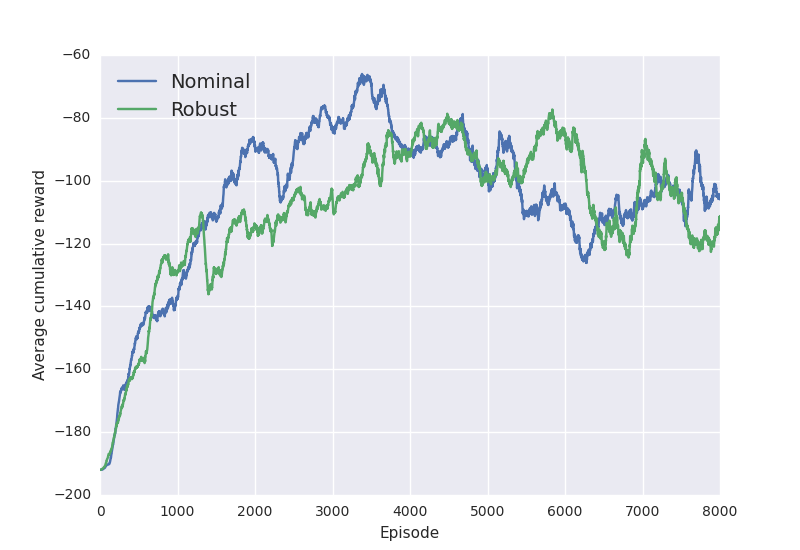}
\end{minipage}%
\begin{minipage}{.32\textwidth}
  \centering
  \includegraphics[scale=0.23]{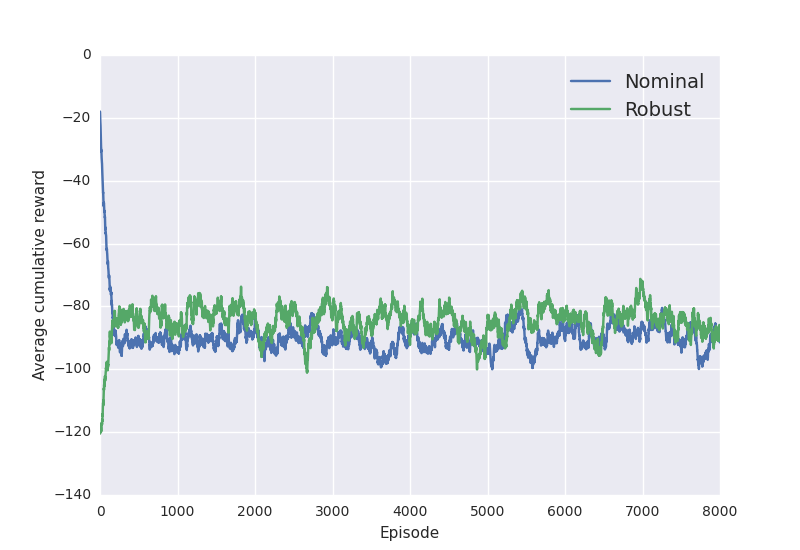}
  \end{minipage}
  \caption{Tail distribution and cumulative rewards during transient and
  stationary phase of robust vs nominal \(\Q\)-learning
  on \textbf{CartPole-v1} with \(p = 0.1\).}
  \label{fig:cartpolev1-1e-1}
\end{figure}

\begin{figure}
\centering
\begin{minipage}{.32\textwidth}
  \centering
  \includegraphics[scale=0.23]{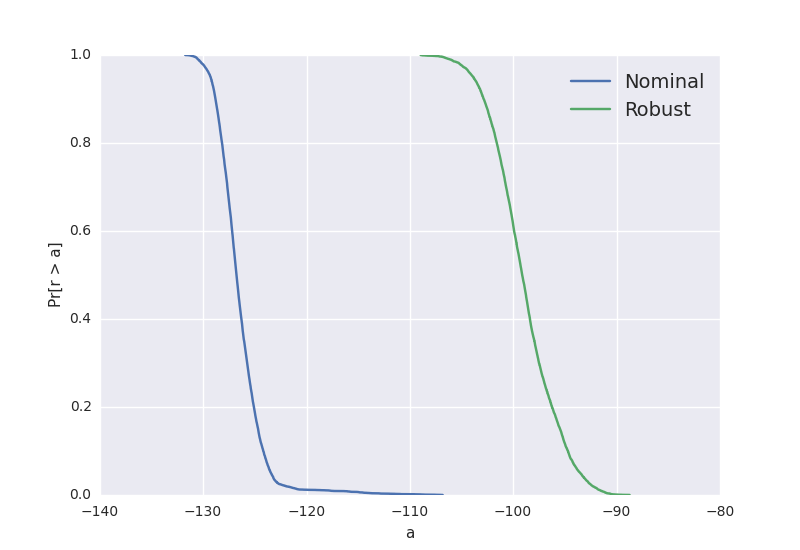}
\end{minipage}%
\begin{minipage}{.32\textwidth}
  \centering
  \includegraphics[scale=0.23]{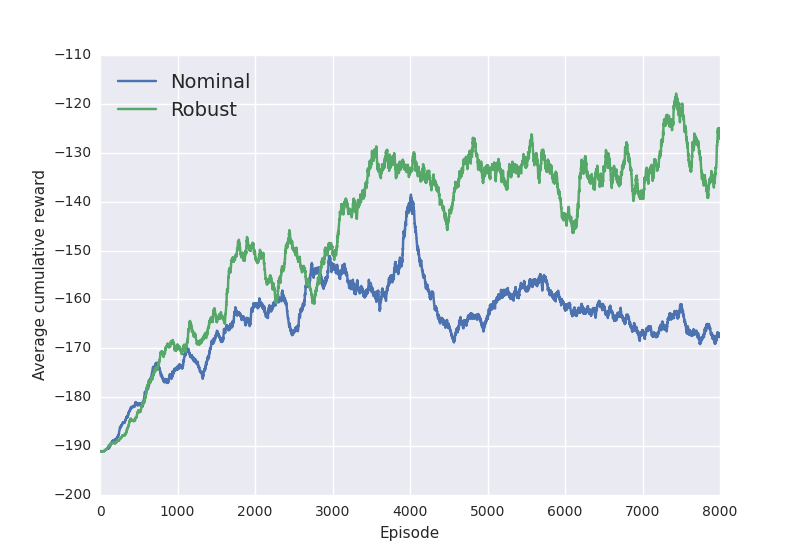}
\end{minipage}%
\begin{minipage}{.32\textwidth}
  \centering
  \includegraphics[scale=0.23]{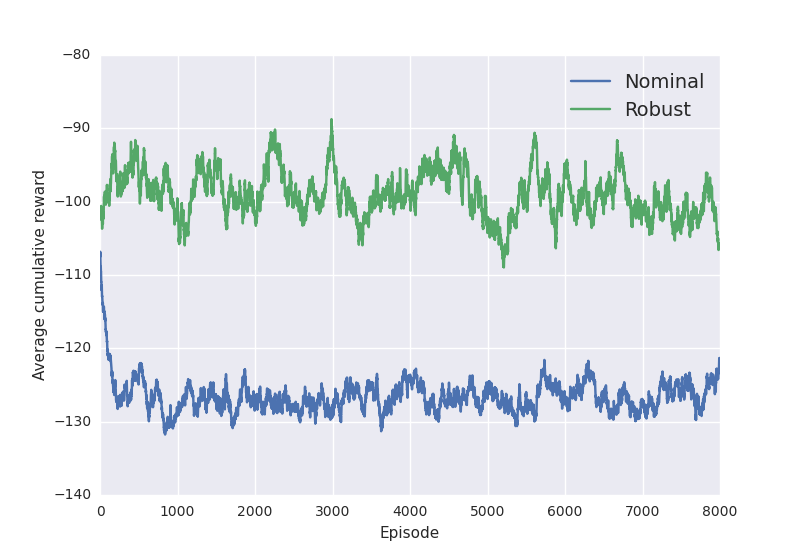}
  \end{minipage}
  \caption{Tail distribution and cumulative rewards during transient and
  stationary phase of robust vs nominal \(\Q\)-learning
  on \textbf{CartPole-v1} with \(p = 0.3\).}
  \label{fig:cartpole-v1-3e-1}
\end{figure}

\begin{figure}
\centering
\begin{minipage}{.32\textwidth}
  \centering
  \includegraphics[scale=0.23]{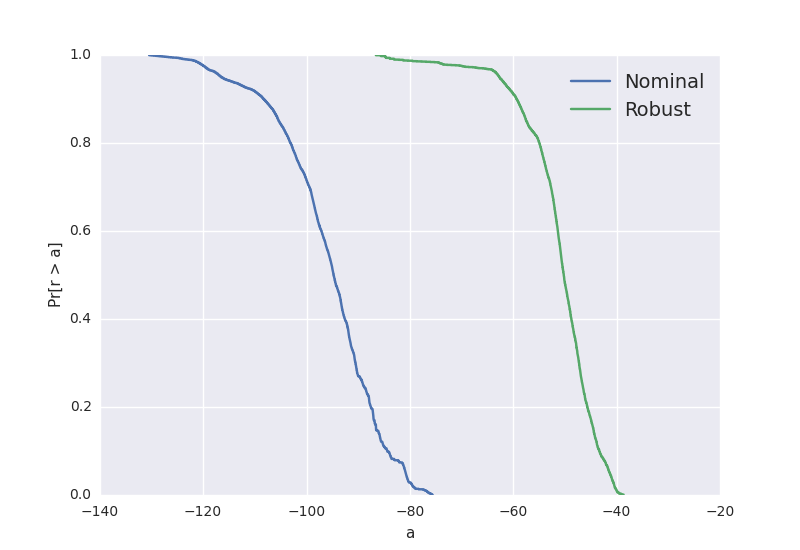}
\end{minipage}%
\begin{minipage}{.32\textwidth}
  \centering
  \includegraphics[scale=0.23]{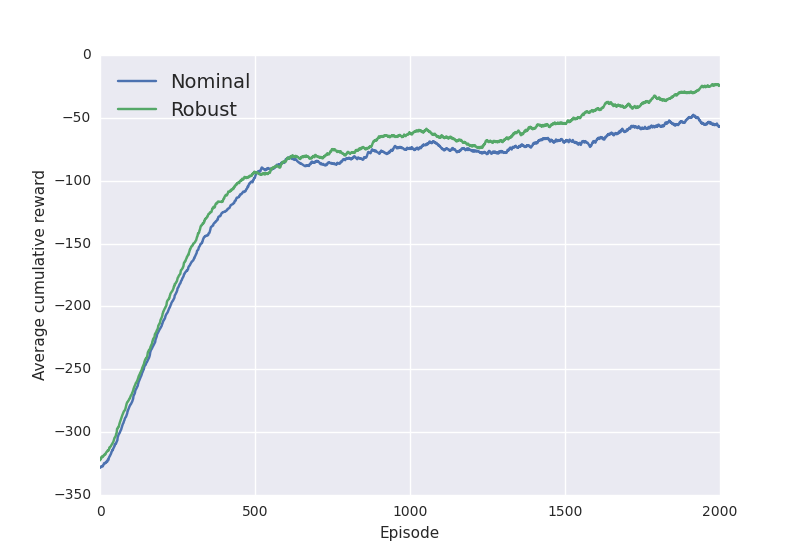}
\end{minipage}%
\begin{minipage}{.32\textwidth}
  \centering
  \includegraphics[scale=0.23]{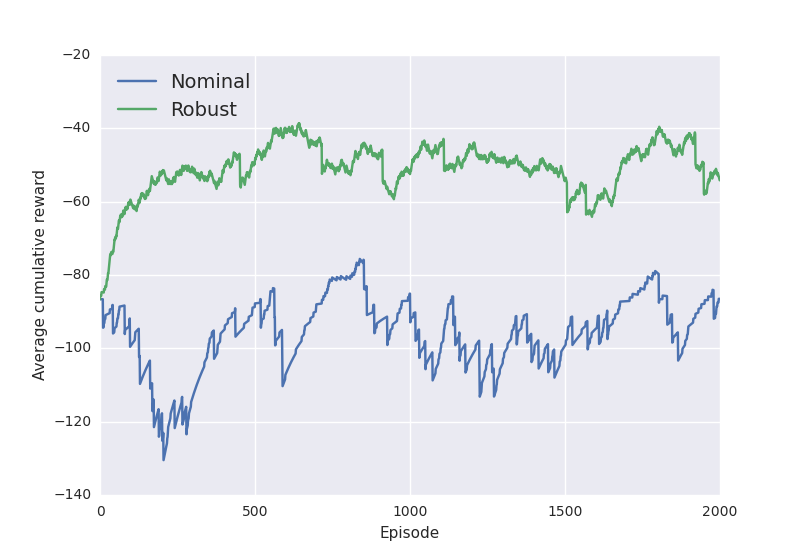}
  \end{minipage}
  \caption{Tail distribution and cumulative rewards during transient and
  stationary phase of robust vs nominal \(\Q\)-learning
  on \textbf{Taxi-v2} with \(p = 0.1\).}
  \label{fig:taxiv2-1e-1}
\end{figure}

\begin{figure}
\centering
\begin{minipage}{.32\textwidth}
  \centering
  \includegraphics[scale=0.23]{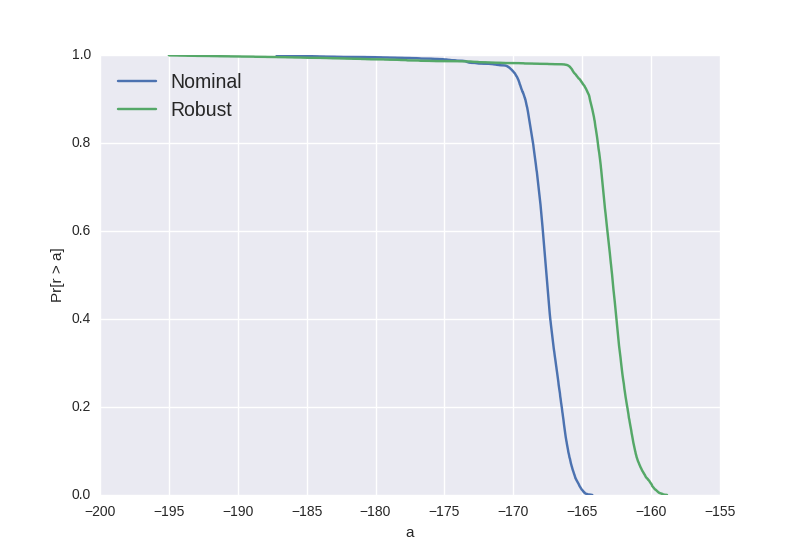}
\end{minipage}%
\begin{minipage}{.32\textwidth}
  \centering
  \includegraphics[scale=0.23]{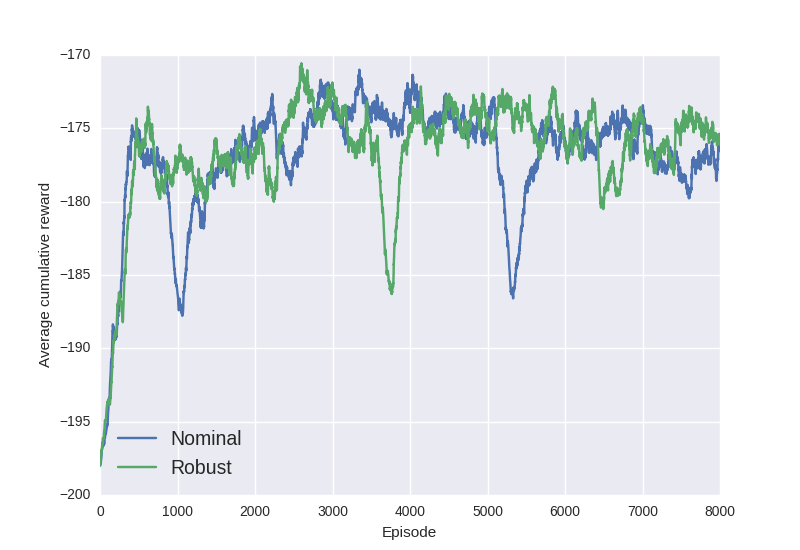}
\end{minipage}%
\begin{minipage}{.32\textwidth}
  \centering
  \includegraphics[scale=0.23]{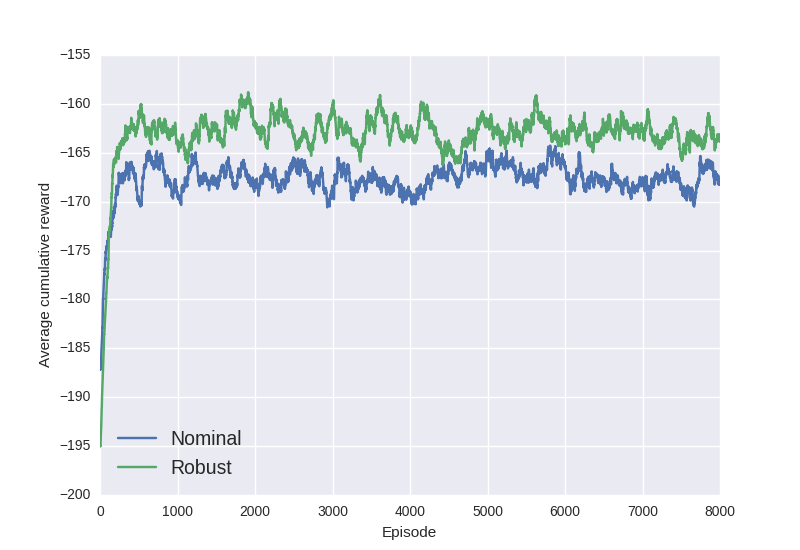}
  \end{minipage}
  \caption{Tail distribution and cumulative rewards during transient and
  stationary phase of robust vs nominal \(\Q\)-learning
  on \textbf{InvertedPendulum-v1} with \(p = 0.1\).}
  \label{fig:invertedpendulum1-0.1}
\end{figure}
\bibliographystyle{plain}

\section{Acknowledgments}
The authors would like to thank Guy Tennenholtz and 
anonymous reviewers for helping improve the 
presentation of the paper.
\bibliography{literature}
\end{document}